\documentclass{article}

\PassOptionsToPackage{numbers, compress}{natbib}

\usepackage[main, final]{neurips_2026}

\usepackage[utf8]{inputenc} %
\usepackage[T1]{fontenc}    %
\usepackage{hyperref}       %
\usepackage{url}            %
\usepackage{booktabs}       %
\usepackage{amsfonts}       %
\usepackage{nicefrac}       %
\usepackage{microtype}      %
\usepackage{xcolor}         %

\usepackage{amsmath,amssymb,amsthm,amsfonts}
\usepackage{bm}
\usepackage{mathtools}
\usepackage{booktabs}
\usepackage{array}
\usepackage{cleveref}

\crefname{theorem}{Theorem}{Theorems}
\Crefname{theorem}{Theorem}{Theorems}
\crefname{lemma}{Lemma}{Lemmas}
\Crefname{lemma}{Lemma}{Lemmas}
\crefname{corollary}{Corollary}{Corollaries}
\Crefname{corollary}{Corollary}{Corollaries}
\crefname{proposition}{Proposition}{Propositions}
\Crefname{proposition}{Proposition}{Propositions}
\crefname{definition}{Definition}{Definitions}
\Crefname{definition}{Definition}{Definitions}
\crefname{remark}{Remark}{Remarks}
\Crefname{remark}{Remark}{Remarks}
\crefname{result}{Result}{Results}
\Crefname{result}{Result}{Results}
\usepackage{multirow}
\usepackage{subcaption}
\usepackage[most]{tcolorbox}
\usepackage{xspace}

\usepackage{aliascnt}

\theoremstyle{plain}
\newtheorem{theorem}{Theorem}[section]
\newaliascnt{corollary}{theorem}
\newtheorem{corollary}[corollary]{Corollary}
\aliascntresetthe{corollary}

\theoremstyle{definition}
\newaliascnt{lemma}{theorem}
\newtheorem{lemma}[lemma]{Lemma}
\aliascntresetthe{lemma}
\newaliascnt{definition}{theorem}
\newtheorem{definition}[definition]{Definition}
\aliascntresetthe{definition}
\newaliascnt{result}{theorem}
\newtheorem{result}[result]{Result}
\aliascntresetthe{result}
\newaliascnt{proposition}{theorem}
\newtheorem{proposition}[proposition]{Proposition}
\aliascntresetthe{proposition}

\theoremstyle{definition}
\newaliascnt{remark}{theorem}
\newtheorem{remark}[remark]{Remark}
\aliascntresetthe{remark}

\newcommand{\DeltaMomentum}{\textsc{AK-Momentum}\xspace}

\newcommand{\vecop}{\operatorname{vec}}

\newcommand{\E}{\mathbb{E}}
\newcommand{\R}{\mathbb{R}}
\newcommand{\tr}{\operatorname{tr}}

\newcommand{\NS}{\operatorname{NS}}
\newcommand{\Sigx}{\Sigma_x}
\newcommand{\Sighat}{\hat{\Sigma}}
\newcommand{\xhat}{\hat{x}}
\newcommand{\uhat}{\hat{u}}
\newcommand{\lamhat}{\hat{\lambda}}
\newcommand{\Gbar}{\bar{G}}

\newcommand{\Mstar}{M^{*}}
\newcommand{\tbeta}{\tilde{\beta}}
\newcommand{\Xhat}{\hat{X}}
\newcommand{\Ghat}{\hat{G}}
\newcommand{\Mhat}{\widehat{M}}
\newcommand{\nablaY}{\nabla Y}
\newcommand{\nablaW}{\nabla W}
\newcommand{\dout}{d_{\mathrm{out}}}
\newcommand{\dinp}{d_{\mathrm{in}}}
\newcommand{\FLOP}{\mathrm{FLOPs}}

\newcommand{\bigO}{\mathcal{O}}
\newcommand{\akmcodeurl}{\url{https://github.com/euijinh/AK-Momentum}}

\title{Activation-Keyed Momentum: An Anisotropic Momentum Update via the Delta Rule}

\author{%
  Euijin Hong \qquad Guannan Qu \\
  Electrical and Computer Engineering, Carnegie Mellon University \\
  \texttt{\{ehong, gqu\}@andrew.cmu.edu}
}

\hypersetup{hidelinks}

\begin{document}

\maketitle

\begin{abstract}
Most modern optimizers form their momentum as an exponential moving average (EMA) of past gradients, forgetting every direction at one fixed rate. 
However, the inputs a deep network sees during training can be highly anisotropic, with a few directions queried frequently while most are seen rarely. 
Preconditioning methods address this anisotropy by wrapping extra processing around this buffer and leave the momentum update itself unchanged. 
We propose Activation-Keyed Momentum (\DeltaMomentum), which builds direction-awareness into the momentum update rule. The gradient of a linear layer splits into an input activation that acts as a \textit{key} and an output-side error that acts as a \textit{value}. Keying on that activation, \DeltaMomentum updates the momentum buffer by the canonical delta rule, so each direction is forgotten at a rate set by how often it appears. 
We prove that it is a valid momentum, that it applies the input-side curvature correction without matrix inversion, and that it clears stale directions faster than EMA under both a fixed and a drifting optimum. 
It is a drop-in replacement for the momentum buffer of any optimizer, its coefficient transfers across widths under $\mu$P, and its extra compute stays between $22.2\%$ and $25.0\%$ of a gated-MLP block's linear cost with no persistent memory. 
In FineWeb-Edu pretraining, AdamW with \DeltaMomentum (AK-AdamW) reaches AdamW's validation loss in up to $46.64 \pm 4.52\%$ fewer steps at 67M and $22.14 \pm 0.81\%$ at 370M over three seeds, and the gain persists at 1B on a Chinchilla-optimal budget. 
A Muon baseline tuned under the same protocol sits above AK-AdamW at both language-model scales, and the gain holds for SGD, ResNet-18, and ViT-Tiny on CIFAR-10. 
Training-time diagnostics confirm the predicted mechanism, better gradient tracking and healthier input directions. Code is available at \akmcodeurl.
\end{abstract}

\section{Introduction}
\label{sec:introduction}

Training large neural networks is expensive, and much of the recent progress in large-scale AI has come from optimizers that reach better solutions in fewer steps, from Adam and AdamW \citep{kingma2014adam, loshchilov2017decoupled} to more recent methods such as Shampoo, SOAP, and Muon \citep{gupta2018shampoo, vyas2024soap, jordan2024muon}. However, underneath all of these optimizers, the momentum remains unchanged. It keeps a running average of past gradients, integrating every new gradient and forgetting older ones at a single fixed rate \citep{polyak1964some}. This exponential moving average (EMA) update rule of momentum is the part of the optimizer that this paper rethinks.

In a linear layer, the gradient of a single data point cleanly separates into two pieces: the input the layer just received and the backpropagated error from the loss to its output. The first piece indicates where the layer received the signal within the input space, and the second says what corresponding correction to apply at the output. In the literature, it is known that the two pieces can be viewed as a key and a value \citep{aizerman1964,irie2022dual,schmidhuber1992,ba2016,schlag2021linear,behrouz2026nested,behrouz2025s}, an address and the content to store at that address, respectively. 
EMA momentum discards this structure. By averaging each gradient as a flat matrix and fading all of them at uniform rate, it collapses the key-value pairing the gradient provides for free, throwing away which directions are being queried and how often, before the optimizer ever leverages it. 

This structure matters because training is largely anisotropic. The input statistics of a layer and the curvature of the loss are dominated by a few directions, followed by a long tail of directions that are nearly flat and rarely excited \citep{sagun2017empirical, ghorbani2019investigation, papyan2019measurements, mu2017all, ethayarajh2019contextual}. A single fixed forgetting rate is problematic at both ends. Along directions that appear often, the newest gradients carry the freshest information, yet the stale gradients are still more retained. Along directions that appear rarely, more averaging would suppress noise, yet the existing information fades shortly even before new gradients arrive. An ideal momentum update would set the forgetting rate by how often a direction is actually seen, which a single rate cannot deliver. Strong optimizers such as K-FAC, Shampoo, and SOAP address this anisotropy by preconditioning, but leave the momentum update itself unchanged, and the goal of this paper is to address it at the momentum level.

\textbf{Contribution.} We propose \DeltaMomentum \eqref{eq:dm-update}, which builds this direction-awareness into the momentum update itself. The method treats the incoming input activation as a key, erases the stored content along the direction the key points to, and writes the new content there before moving on. Directions queried often are refreshed often and forgotten quickly, while directions queried rarely are kept over long horizons, so the forgetting along each direction now follows how often that direction is seen. The rule behind this is the classical delta rule \citep{widrow1988adaptive}, the canonical method for storing key-value associations and the mechanism of recent associative-memory sequence models \citep{schlag2021linear, yang2024parallelizing}, applied here to the momentum buffer with the input activation as the key. 

The design is not only intuitive but provably well-behaved, and we establish three guarantees. First, \DeltaMomentum remains a valid momentum, an estimator of the gradient with a closed-form fixed point (\Cref{thm:fixedpoint}). Second, that fixed point applies the input-side curvature correction that curvature-based methods obtain by inverting a matrix, without inverting one (\Cref{rem:kfac}). Third, it clears stale directions strictly faster than EMA, both when the target is fixed (\Cref{thm:detred}) and when it drifts as the weights move (\Cref{thm:tracking}), which is what happens in practical training.

Two further properties make \DeltaMomentum practical. First, it is a drop-in replacement for the momentum in any optimizer, leaving the second moment, weight decay, and any preconditioning exactly as they were. Therefore, it composes with Shampoo, SOAP, and Muon rather than competing with them. Second, using normalized input as the key further stabilizes the training and establishes a scaling law under maximal-update parameterization ($\mu$P) \citep{yang2021tuning} whose single new coefficient transfers across widths without re-tuning and whose extra compute is small and well-characterized.

We test these advantages in language-model pretraining at 67M, 370M, and 1B parameters on FineWeb-Edu \citep{penedo2024fineweb}, against AdamW, the baseline that differs from our method only in the momentum update. AdamW with \textsc{AK-Momentum} (AK-AdamW) reaches AdamW's validation loss in up to $46.64 \pm 4.52\%$ fewer steps at 67M and $22.14 \pm 0.81\%$ fewer at 370M and continues to a lower final loss, a gain that comfortably exceeds its compute overhead. The gap remains positive in every run and holds at 1B on a Chinchilla-optimal budget with every setting carried over from the 67M sweep, and AK-AdamW also sits below a tuned Muon baseline. The same pattern holds for SGD, ResNet-18, and ViT-Tiny on CIFAR-10, so the effect is not tied to Adam or to language, and training-time diagnostics confirm the predicted mechanism, better gradient tracking and healthier input directions.

\textbf{Related Work.} Three lines of work are closely related to ours. 

\textit{Preconditioning-based optimizers.} K-FAC, Shampoo, and SOAP precondition the gradient using curvature information, and Muon post-processes the momentum to balance its directions \citep{martens2015optimizing,gupta2018shampoo,vyas2024soap,jordan2024muon}. Each keeps the running-average momentum and wraps extra processing around it, so \DeltaMomentum, which changes the averaging rule itself, can be combined with all of them.

\textit{Input-side preconditioning.} FOOF \citep{benzing2022gradient} preconditions the gradient with the inverse second moment of the input activations alone, and DoPr and Newton-Muon \citep{zhang2026dopr,du2026newtonmuon} apply such a preconditioner to the momentum. That is the same input-side factor \DeltaMomentum reaches implicitly, but \DeltaMomentum does not need matrix inversion or any additional persistent memory beyond the momentum that the base optimizer already keeps.

\textit{Optimizers as associative memory.} Viewing a layer trained by gradient descent as an associative memory has a long history. Its weight is the initialization plus a sum of outer products of the inputs it received and the errors it was trained on, so its output is an unnormalized attention read over its own training history \citep{aizerman1964,irie2022dual}, an identity that fast weight programmers also build on \citep{schmidhuber1992,ba2016,schlag2021linear}. \citet{behrouz2026nested,behrouz2025s} carry this view to the optimizer and its momentum, with the input activation as the key and the local error as the value. However, their motivation is memory capacity over long horizons rather than feature anisotropy. Their buffer-level rule keys on the gradient object rather than on its two factors, so a direction is forgotten at a rate set by the size of the error rather than how often that direction is seen. \DeltaMomentum keys on the input activation, which is what binds the forgetting rate to input density. \citet{wang2026taming} act at the same level but compress the buffer by low-rank approximation. Appendix B compares the update rules side by side.

\section{Background}
\label{sec:preliminaries}

In step $t$, the optimizer holds parameters $\theta_t$ and momentum $M_{t-1}$, computes the gradient $g_t = \nabla_\theta \mathcal{L}_t(\theta_t)$, and updates the momentum to $M_t$ and the parameters to $\theta_{t+1}$. EMA momentum updates
\begin{align}
M_t^{\mathrm{EMA}} &= \beta M_{t-1}^{\mathrm{EMA}} + (1-\beta)\, g_t,
\label{eq:ema-update}
\end{align}
with momentum decay $\beta \in [0,1)$. 
For a linear layer $y = Wx$ with weight $W \in \R^{m\times n}$, input
$x \in \R^n$, and output-side error $\delta = \partial L / \partial y \in \R^m$,
the per-sample gradient \emph{factorizes} as rank-1 outer product
\begin{equation}
g_t = \delta_t x_t^\top \in \R^{m\times n}.
\label{eq:rank-one-grad}
\end{equation}
This factored form is the structural fact that the rest of the paper exploits. It identifies $x_t$ as a natural \emph{key} and $\delta_t$ as a natural \emph{value}, and makes the momentum buffer an associative memory of past (key, value) pairs. 

\textbf{The Delta Rule}
\label{subsec:delta-rule}
The classical delta rule \citep{widrow1988adaptive} updates a matrix $M \in \R^{m\times n}$ that associates keys $k \in \R^n$ with values $v \in \R^m$, with step size $\eta$, by
\begin{equation}
M \leftarrow M + \eta\,(v - M k)\, k^\top,
\label{eq:delta-rule}
\end{equation}
where $v - Mk$ is the \emph{retrieval error}. 
Each step is one online stochastic gradient step on the prediction loss $\tfrac12\|v - Mk\|_2^2$, so $M$ solves the least-squares problem $\min_M \sum_{\ell} \tfrac12\|v_\ell - M k_\ell\|_2^2$ over the stored pairs online. A purely Hebbian accumulator $M_t = M_{t-1} + v_t k_t^\top$ accumulates interference. Fast weight programmers for sequence modeling instead use the delta rule, $M_t = M_{t-1}(I_n - \eta k_t k_t^\top) + \eta v_t k_t^\top$, which removes the old association at $k_t$ before writing the new one and provably preserves the orthogonality structure of the stored memory \citep{schlag2021linear,yang2024parallelizing}.

\section{\DeltaMomentum: Methodology}
\label{sec:methodology}

\label{subsec:def}
\textbf{Rationale: online key--value gradient memory.} \DeltaMomentum applies the delta rule (3) to the momentum buffer, with the input activation as the key, $k_t = x_t$, and the output-side error as the value, $v_t = \delta_t$. A uniform decay $\beta$ keeps an EMA-style memory along directions no key touches.

\begin{tcolorbox}[
  width=\linewidth,
  colback=white,
  colframe=gray!60!black
]
\textbf{\textsc{\DeltaMomentum update rule.}}
\begin{equation}
M_t = \beta M_{t-1} + \eta\,\bigl(\delta_t - M_{t-1} x_t\bigr)\, x_t^\top,
\label{eq:dm-update}
\end{equation}
\end{tcolorbox}

with EMA decay $\beta \in [0, 1)$ and delta-rule correction coefficient $\eta \in (0, 1]$. 
Setting $\beta = 1$ recovers the pure delta rule. 
Factoring out $M_{t-1}$ gives the form used throughout the analysis, 
\begin{equation}
M_t = M_{t-1}\bigl(\beta I_n - \eta\, x_t x_t^\top\bigr) + \eta\,\delta_t x_t^\top.
\label{eq:dm-factored}
\end{equation}

The right-multiplying operator $\beta I - \eta x_t x_t^\top$ is a data-dependent forgetting factor that contracts directions aligned with the current key and leaves orthogonal directions untouched, while $\eta\delta_t x_t^\top$ writes the new association. Rather than averaging gradients as flat matrices, \DeltaMomentum stores the associations $(x_t, \delta_t)$ in $M$ and updates them by the delta rule. Equation \eqref{eq:dm-factored} is the conceptual centerpiece, since every property that follows is a property of this operator: the fixed point, the memory horizons, and the tracking dynamics.

\textbf{Modular instantiations.}
\DeltaMomentum replaces the first-moment buffer of any base optimizer. SGD applied with \DeltaMomentum (AK-SGD) uses~\eqref{eq:dm-update} as the momentum,
followed by $\theta_{t+1} = \theta_t - \alpha M_t$ with learning rate $\alpha$. 
AdamW applied with \DeltaMomentum (AK-AdamW) applies~\eqref{eq:dm-update} to the first moment and leaves the second moment $v_t = \beta_2 v_{t-1} + (1-\beta_2) g_t^{\odot 2}$, bias correction, and decoupled weight decay unchanged, with preconditioned update $U_t = \hat{M}_t \oslash (\sqrt{\hat{v}_t} + \varepsilon)$. 

\textbf{$\mu$P-clean width transfer.}
\label{par:mup-main}
\DeltaMomentum is $\mu$P-compatible. With the normalized key, the spectral norm of the rank-1 forgetting operator is $\|\hat x_t \hat x_t^\top\|_{\mathrm{op}} = 1$ at every width, so $\eta$ is width-invariant. 
Under $\mu$P, \eqref{eq:deployed} has a fixed-point per-coordinate scale of $\Theta(n^{-1/2})$, the scale EMA buffer attains and the base learning-rate rule expects, if and only if $\eta = \Theta(1)$ (Appendix~\ref{app:muP}).
That rule is therefore inherited unchanged, with $\alpha = \Theta(1/n)$ for AK-AdamW and $\Theta(1)$ for AK-SGD on hidden layers. The result is zero-shot hyperparameter transfer across widths, which a coordinate check at widths 128 to 2048 confirms (Appendix~\ref{app:mup-coord}), and Appendix~\ref{app:hp-sweep} gives the recommended recipe for the one new coefficient.

\textbf{FLOPs and memory.}
\label{par:flops-memory}
Against EMA \eqref{eq:ema-update}, the only extra work of the deployed rule is the correction $\eta M\hat\Sigma$, where $\hat\Sigma$ is the second moment of the normalized keys, estimated on each batch. Its gradient term is the weight gradient that the backward pass already computes (Appendix~\ref{app:deployed}). Per transformer block, with the cheapest contraction order for each layer, the correction costs $(6+4\gamma)/(24+18\gamma)$ of the linear forward and backward, where $\gamma$ is the gated-MLP expansion ratio. This share decreases in $\gamma$, so it stays between $22.2\%$ and $25.0\%$ for $\gamma \ge 2$ and is $23.1\%$ at the Llama-2 ratio we use. Counted by our training suite, the realized per-step overhead is $10.0\%$ at 67M, $15.5\%$ at 370M, and $18.1\%$ at 1B. These sit below the per-block band, since attention carries no delta-rule work and the output head's correction is small next to its own cost, and they rise with width toward it. Appendix~\ref{app:deployed} gives the batch form of the deployed rule and Appendix~\ref{app:flops-memory} the full count and measurement setup. Persistent optimizer-state memory overhead is zero, and transient memory is bounded by $4Nn_{\max}$ bytes per block for a microbatch of $N$ tokens and largest layer input width $n_{\max}$.

We organize the remaining parts of this section as follows.
Section~\ref{sec:gradient-estimator} shows why \DeltaMomentum is a valid momentum. 
Section~\ref{sec:transport} proves the lemma that explains why it is better than EMA.   
Sections~\ref{sec:linear-regression-dynamics} and \ref{sec:non-stationary-tracking-dynamics} derive the two consequences from that lemma: faster per-direction convergence and direction-selective tracking. 
Appendix~\ref{app:symbols-and-notations} lists all notation.

\subsection{\DeltaMomentum is a Valid Gradient Estimator}
\label{sec:gradient-estimator}

Any momentum buffer is, to first order in the learning rate, an estimator of the population gradient $\Gbar(\theta_t) = \nabla_\theta \mathcal{L}(\theta_t)$, the gradient of the expected loss over the data distribution rather than that of any single minibatch. 
We identify the loss whose minimization defines \DeltaMomentum and characterize its fixed point.

\begin{remark}[Normalized-key substitution]
\label{rem:normalized-key}
The results from now on are stated for the unnormalized rule \eqref{eq:dm-factored}, whose key has second moment $\Sigma_x = \mathbb{E}[xx^\top]$ with eigenvalues $\lambda_i$, and use the key only through that second moment. 
So, they hold for the normalized-key rule \eqref{eq:dm-norm-update} under $\Sigma_x \to \hat\Sigma$, $\bar G \to \hat G$, and $\lambda_i \to \hat\lambda_i$. 
Since $\Sighat$ has unit trace, $\lamhat_i \in [0,1]$ measures the directional density of the input rather than its energy (\Cref{prop:bounded-sigma-hat}). 
The convergence point of \eqref{eq:dm-fixedpoint-norm} and the trace of \eqref{eq:iter-matrices} need more than substitution. 

\end{remark}

\textbf{Momentum as a gradient estimator.}
\label{par:mom-grad-estimator} 
For any realization of $M_t$, a first-order Taylor expansion gives $\mathcal{L}(\theta_t) - \mathcal{L}(\theta_t - \alpha M_t) = \alpha\langle M_t, \bar G(\theta_t)\rangle_F + O(\alpha^2)$, so a buffer that deviates from $\bar G_t$ pays a first-order penalty $\alpha\langle \bar G_t - M_t, \bar G_t\rangle_F$, and the quality of $M_t$ as an estimator of $\bar G_t$ governs optimization efficiency (Appendix~\ref{app:taylor}).
What distinguishes EMA from \DeltaMomentum is therefore not whether they estimate the gradient, but \emph{which loss the buffer minimizes online} to perform that estimate.

\textbf{Two natural objectives, two distinct fixed points.}
\label{par:two-fixed-points}
The factored gradient $g_t = \delta_t x_t^\top$ admits two natural mean-square estimation objectives.

\begin{definition}[Weight-space MSE]\label{def:wms}
$J_w(M) := \tfrac{1}{2}\mathbb{E}\bigl[\|M - \delta x^\top\|_F^2\bigr]$.
\end{definition}
\begin{definition}[Function-space prediction]\label{def:fms}
$J_f(M) := \tfrac{1}{2}\mathbb{E}\bigl[\|\delta - M x\|_2^2\bigr]$.
\end{definition}

Under stationarity, EMA \eqref{eq:ema-update} is online stochastic gradient descent on $J_w$ and converges in expectation to $\bar G$, an unbiased estimator of the population gradient. The pure delta rule is online stochastic gradient descent on $J_f$ and, when $\Sigma_x \succ 0$, converges to $\bar G\Sigma_x^{-1}$, the Wiener-optimal linear predictor of $\delta$ given $x$. 
The two fixed points agree up to a scalar for every $\bar G$ if and only if $\Sigma_x = cI$ (Appendices~\ref{app:divergent-optima} and \ref{app:updates-as-sgd}).

\textbf{The Tikhonov-regularized Wiener fixed point of \DeltaMomentum.}
\label{par:fixed-point}
Adding the uniform decay $\beta < 1$ regularizes this Wiener predictor.

\begin{theorem}[Fixed point of \DeltaMomentum]\label{thm:fixedpoint}
Under the quasi-static approximation (parameters change slowly relative to the
momentum dynamics, so $\bar{G}$ and $\Sigma_x$ are treated as constant) and \(\rho(\beta I-\eta\Sigma_x)<1\), the \DeltaMomentum
update~\eqref{eq:dm-update} converges in expectation to
\begin{equation}
M^\star = \eta\, \bar{G}\,\bigl((1-\beta) I + \eta \Sigma_x\bigr)^{-1}.
\label{eq:dm-fixedpoint}
\end{equation}
In the eigenbasis of $\Sigma_x$,
\begin{equation}
M^\star u_i = \phi_i\, \bar{G} u_i,
\qquad
\phi_i = \frac{\eta}{(1-\beta) + \eta \lambda_i} = \frac{1}{\mu + \lambda_i},
\qquad
\mu := \frac{1-\beta}{\eta}.
\label{eq:dm-eigen}
\end{equation}
\end{theorem}
\begin{proof}
Take expectations of \eqref{eq:dm-factored}. At stationarity
$M^\star\bigl((1-\beta)I + \eta \Sigma_x\bigr) = \eta \bar{G}$. See
Appendix~\ref{app:fixed-point}.
\end{proof}

\begin{corollary}[Fixed point of the deployed algorithm]
\label{cor:fixedpoint-normalized}
For the normalized-key update the same computation gives
\begin{equation}
\Mstar = \eta\,\Ghat\,\bigl((1-\beta)I + \eta\Sighat\bigr)^{-1} = \Ghat\,(\mu I + \Sighat)^{-1},
\qquad
\Mstar u_i = \frac{\Ghat u_i}{\mu + \lamhat_i}.
\label{eq:dm-fixedpoint-norm}
\end{equation}
\end{corollary}

When the key magnitude $\|x\|$ is independent of $(\hat x, \delta)$, \eqref{eq:dm-fixedpoint-norm} is the unnormalized fixed point with a rescaled delta step, and the leftover scalar is absorbed by the learning rate. This holds exactly when $\|x\|$ is deterministic, as at every linear layer of a pre-normalization transformer that reads directly from an RMSNorm output \citep{touvron2023llama}. Elsewhere the buffer weights directions of larger input norm more heavily, and none of the results below depends on the condition, since they use the key only through its second moment. For the deployed rule \eqref{eq:deployed}, $\Ghat$ in \eqref{eq:dm-fixedpoint-norm} becomes $\Gbar$.

\begin{table}[!ht]
\centering
\caption{Three estimator targets and their per-direction scaling of $\bar{G}u_i$
in the eigenbasis of $\Sigma_x$. The Tikhonov ridge is $\mu = (1-\beta)/\eta$.}
\label{tab:targets}
\begin{tabular}{lll}
\toprule
Estimator & $M^\star u_i / \bar{G}u_i$ & Description \\
\midrule
EMA & $1$ & Unbiased estimator of $\bar{G}$ \\
Wiener ($\bar{G}\Sigma_x^{-1}$) & $1/\lambda_i$ & Optimal linear predictor of $\delta$ given $x$ \\
\DeltaMomentum & $\phi_i = 1/(\mu + \lambda_i)$ & Tikhonov-regularized Wiener \\
\bottomrule
\end{tabular}
\end{table}

The fixed point \eqref{eq:dm-fixedpoint} interpolates monotonically between two limits. As $\mu \to 0$ (i.e., $\beta \to 1$), $\phi_i \to 1/\lambda_i$ recovers the unregularized Wiener predictor, and as $\mu \to \infty$ (i.e., $\eta \to 0$), $\phi_i/\phi_j \to 1$ recovers the direction-uniform scaling of EMA. \textsc{AK-Momentum} thus sits between EMA (direction-uniform) and Wiener (input-whitening), parameterized by the single ratio $\mu$, as shown in Table~\ref{tab:targets}.

\begin{remark}[K-FAC connection]\label{rem:kfac}
The Kronecker-factored Fisher approximation of a linear layer is $F \approx \E[\delta\delta^\top] \otimes \E[xx^\top]$, with natural-gradient update $\E[\delta\delta^\top]^{-1}\bar{G}\Sigma_x^{-1}$. The pure delta-rule fixed point $\bar{G}\Sigma_x^{-1}$ is exactly the input-side factor of this Kronecker product, so \DeltaMomentum implements \emph{implicit input-side natural-gradient preconditioning} without inverting $\Sigma_x$, at the cost of a single rank-1 update per step rather than inverting a layer-wise covariance matrix. The deployed algorithm whitens $\Sighat$ rather than $\Sigx$, so this identification holds up to a scalar factor, given in \Cref{cor:fixedpoint-normalized}.
\end{remark}

Regarding \eqref{eq:dm-fixedpoint-norm} as a predictor of $\delta$ from $\hat x$ or as the gradient $\Ghat$ right-multiplied by the preconditioner $(\Sighat + \mu I)^{-1}$ describes the same object in two coordinate systems (Appendix~\ref{app:estimator-and-preconditioner}), and Appendix~\ref{app:second-moment} discusses how the delta correction interacts with the second moment of Adam within AK-AdamW.

\subsection{Anisotropic Memory Transport}
\label{sec:transport}

Iterating \eqref{eq:dm-factored} from $M_0 = 0$ expands $M_t$ as a sum of past error-key outer products, each transported forward in time by a product of data-dependent contractions,
\begin{equation}
M_t = \eta \sum_{\ell=1}^{t} \delta_\ell x_\ell^\top
\underbrace{\prod_{j=\ell+1}^{t}\bigl(\beta I_n - \eta x_j x_j^\top\bigr)}_{P_{\ell\to t} \;\in\; \mathbb{R}^{n\times n}},
\label{eq:transport-expansion}
\end{equation} 
with $P_{t\to t} = I_n$. 
Each term $\eta \delta_\ell x_\ell^\top P_{\ell\to t}$ is the fading memory of the association stored at step $\ell$. 
Along directions orthogonal to every subsequent key, $P_{\ell\to t}$ reduces to $\beta^{t-\ell} I$, a uniform geometric decay. Along directions aligned with subsequent keys, the additional factor $-\eta x_j x_j^\top$ applies a direction-selective attenuation, erasing stale content that conflicts with newly arriving keys.

\begin{lemma}[Anisotropic memory transport]\label{lem:transport}
Treat $\{x_j\}_{j > \ell}$ as drawn independently from a distribution with zero
mean and covariance $\Sigma_x = \mathbb{E}[x x^\top] = \sum_i \lambda_i u_i u_i^\top$ where $\lambda_i$'s are the eigenvalues and $u_i$'s are the corresponding unit eigenvectors.
Then
\begin{equation}
\mathbb{E}[P_{\ell \to t}]\, u_i = (\beta - \eta \lambda_i)^{t-\ell} u_i
= \tilde{\beta}_i^{\,t-\ell} u_i,
\qquad
\tilde{\beta}_i := \beta - \eta \lambda_i.
\label{eq:transport-lemma}
\end{equation}
\end{lemma}
\begin{proof}
By independence,
$\mathbb{E}[P_{\ell\to t}] = \prod_{j=\ell+1}^{t}\mathbb{E}[\beta I - \eta x_j x_j^\top]
= (\beta I - \eta \Sigma_x)^{t-\ell}$. Eigendecomposing $\Sigma_x$ gives
\eqref{eq:transport-lemma}.
\end{proof}

Lemma~\ref{lem:transport} states that the expected memory of the buffer along
direction $u_i$ decays geometrically at rate $\tilde{\beta}_i$, with memory
horizon
\begin{equation}
\tau_i = \frac{1}{1 - \tilde{\beta}_i} = \frac{1}{(1-\beta) + \eta \lambda_i}.
\label{eq:horizon}
\end{equation}
High-density directions (large $\lambda_i$), where the local gradient changes quickly with $\theta$, are forgotten on a short horizon. Nearly stationary low-density directions ($\lambda_i \approx 0$) keep an EMA-like horizon $\approx 1/(1-\beta)$. This is the data-driven schedule the function-space prediction loss prescribes, and this direction-dependent memory horizon is the unifying mechanism behind every result that follows.

\textbf{Lemma~\ref{lem:transport} implies three advantages.}
\label{par:anisotropic-memory-transport-implications}
\Cref{thm:fixedpoint} is its steady-state implication, with steady-state ratio of injection to forgetting $\eta / [(1-\beta) + \eta \lambda_i] =\eta\tau_i$ along $u_i$, and \Cref{sec:linear-regression-dynamics,sec:non-stationary-tracking-dynamics} derive the per-direction and non-stationary advantages as direct corollaries.

\subsection{Faster Per-Direction Convergence}
\label{sec:linear-regression-dynamics}

We compare how fast \DeltaMomentum and EMA drive the parameter error to zero along each input direction. 
In linear regression the comparison is exact, and the per-direction iteration matrix of \DeltaMomentum has a strictly smaller determinant than that of EMA, which yields faster per-direction convergence in the underdamped regime. 
The determinant never uses the squared loss, so the statement holds beyond that setting as well.

\textbf{Setup}
Consider the population objective
$\mathcal{L}(W) = \tfrac{1}{2}\mathbb{E}[\|Wx - y\|^2]$ with $y = W_\star x$ and
$x \sim \mathcal{D}_x$ with covariance $\Sigma_x$. Letting $E_t = W_t - W_\star$,
the dynamics decouple along the eigenbasis of $\Sigma_x$, where we let $e_{t,i} = E_t u_i \in \R^m$, $\mu_{t,i} = M_t u_i \in \R^m$, into per-direction
$2 \times 2$ iteration matrices acting on $(e_{t,i}, \mu_{t-1,i})^\top$: 

\begin{equation}
A_i^{\mathrm{EMA}} = \begin{pmatrix}
1 - \alpha(1-\beta)\kappa_i & -\alpha\beta \\
(1-\beta)\kappa_i & \beta
\end{pmatrix},
\qquad
A_i^{\mathrm{Delta}} = \begin{pmatrix}
1 - \alpha\eta\kappa_i & -\alpha\tilde{\beta}_i \\
\eta\kappa_i & \tilde{\beta}_i
\end{pmatrix},
\label{eq:iter-matrices}
\end{equation}
where $\tbeta_i = \beta - \eta\lamhat_i$ and $\kappa_i = u_i^\top \E[\|x\|\,\xhat\xhat^\top] u_i$, assuming $\Sighat$ and $\E[\|x\|\,\xhat\xhat^\top]$ share eigenvectors (Appendix~\ref{app:iteration-matrices}). 
Normalizing the key splits one parameter into two, since the buffer forgets at a rate set by the directional density $\lamhat_i$ while the parameter error still enters through the raw input, whose per-direction size is $\kappa_i$. 
The split is confined to the trace, so the determinant depends only on $\lamhat_i$, and we write the EMA row with the same coupling so the two rows are read side by side.

From this we derive the determinant reduction per-direction of $A_i^{\mathrm{Delta}}$ over $A_i^{\mathrm{EMA}}$. 
\begin{theorem}[Per-direction determinant reduction]\label{thm:detred}
For all $\eta > 0$ and $\lamhat_i > 0$,
$\det(A_i^{\mathrm{Delta}}) = \tbeta_i = \beta - \eta\lamhat_i < \beta = \det(A_i^{\mathrm{EMA}}),$
and neither determinant depends on the coupling $\kappa_i$.
\end{theorem}

\begin{proof}
Direct computation is shown in Appendix~\ref{app:determinant-reduction}. The result is a direct
consequence of Lemma~\ref{lem:transport}, since the expected per-direction memory
contraction is exactly $\tilde{\beta}_i$.
\end{proof}

\textbf{Extension beyond Linear Regression.} The loss enters \eqref{eq:iter-matrices} only at the coupling $\eta\kappa_i$ of the parameter error into the buffer, where $\kappa_i$ is the curvature of the squared loss along $u_i$.

\begin{lemma}[The determinant does not use the loss]\label{lem:loss-free-det-main}
Let $A$ be the per-direction iteration matrix \eqref{eq:iter-matrices} of either rule with $\eta\kappa_i$ replaced by any constant $c > 0$, implicated as $\eta$ times the local curvature along $u_i$ under whatever loss is minimized. 
Then $\det A = \tbeta_i$ for \DeltaMomentum and $\det A = \beta$ for EMA momentum, for every $c$ and every step size $\alpha$. Consequently the best per-direction rate attainable over step sizes falls from the square root of $\beta$ to the square root of $\tbeta_i$ on every direction with $c > 0$, whatever loss produced it.
\end{lemma}

Appendix~\ref{app:lemma-a} proves the lemma, and Appendices \ref{app:high-dim} and \ref{app:glm} use it to extend \Cref{thm:detred} to high-dimensional linear regression with more input dimensions than samples and to generalized linear models including logistic regression. We state the rate over step sizes because at one fixed step size a smaller determinant does not by itself force a smaller spectral radius. \Cref{thm:strengthened-rate} shows that it does whenever the iteration is underdamped, a band condition on $\alpha\eta\kappa_i$ (\Cref{rem:underdamped-typicality}).

\subsection{Non-stationary Tracking Dynamics}
\label{sec:non-stationary-tracking-dynamics}

The fixed-point analysis of \Cref{par:two-fixed-points} and the linear-regression dynamics of \Cref{sec:linear-regression-dynamics} both assume stationarity. In practical training, the population gradient $\bar{G}_t = \nabla_\theta \mathcal{L}(\theta_t)$ and the input-feature covariance $\Sigma_x(\theta_t) = \E[x_t x_t^\top]$ drift as the weights update. We now characterize the rate at which the momentum buffer adapts to such drift, and Appendix~\ref{app:scope} states which of the assumptions of \Cref{sec:methodology} describe a training run.

\begin{theorem}[Direction-selective tracking under distribution shift]
\label{thm:tracking}
Suppose at step $t_0$ the population gradient and input-feature covariance shift to new fixed values $(\Gbar_k, \Sigma_k)$, and let $\Mstar$ be the corresponding new fixed point. 
Define the tracking error $T_{t,i} = (M_t - \Mstar) u_i$ along eigendirection $u_i$ of $\Sigma_k$. 
Then 
\begin{align}
\text{EMA:} \quad & \E[T_{t,i} \mid T_{t-1,i}] = \beta T_{t-1,i}, \quad \forall i, \label{eq:tracking-ema}\\
\text{Delta:} \quad & \E[T_{t,i} \mid T_{t-1,i}] = \tbeta_i T_{t-1,i} = (\beta - \eta\lambda_i) T_{t-1,i}. \label{eq:tracking-delta}
\end{align}
EMA contracts uniformly at rate $\beta$. 
\DeltaMomentum contracts along $u_i$ at rate $\tbeta_i$, strictly faster for every direction with $\lambda_i > 0$.
\end{theorem}

\Cref{thm:tracking} follows immediately from Lemma~\ref{lem:transport}, holds for any $\beta$ and $\eta$ in the stable range, and applies in expectation along every direction (Appendix~\ref{app:tracking}). 
Unrolling \eqref{eq:tracking-delta} gives the memory horizons \eqref{eq:horizon}, shorter than the uniform EMA horizon by a factor of $1 + \eta\lambda_i/(1-\beta)$, and Appendix~\ref{app:uniform-tracking-rate} gives tighter direction-uniform bounds and density-weighted mean contraction rates. Since the theorem constrains only the recursion that produces the buffer and places no condition on the loss landscape, it survives into non-convex training, where it sharpens into a lag bound.

\begin{proposition}[Steady lag under drift]\label{prop:nc1-main}
If the target of the buffer drifts by at most $\Delta_i$ per step along $u_i$, then $\limsup_t \|\E[T_{t,i}]\| \le \tbeta_i\Delta_i/(1-\tbeta_i)$, against $\beta\Delta_i/(1-\beta)$ for EMA. Since $\tbeta_i < \beta$ on every direction with positive input density and $t \mapsto t/(1-t)$ is increasing, the bound is strictly smaller there, and the relaxation-time ratio is $1 + \eta\lamhat_i/(1-\beta)$.
\end{proposition}

Convexity never enters, so \Cref{prop:nc1-main} holds for the deep non-convex losses of \Cref{sec:experiments}. Appendix~\ref{app:nonconvex} connects the reduced lag back to the first-order loss decrease of \Cref{prop:first-order-optimal}. It is a tracking result and gives no global convergence rate on non-convex losses.

\section{Experiments}
\label{sec:experiments}

\Cref{sec:methodology} predicts implicit input-side preconditioning at the fixed point (\Cref{thm:fixedpoint}), direction-density-aware memory transport (Lemma~\ref{lem:transport}), and a small constant FLOP overhead (\Cref{thm:deployed-overhead}). 
We test whether these claims hold in deep, non-linear, non-stationary training. 
\Cref{sec:main} compares AK-AdamW with AdamW in Llama-2-style language-model pretraining at 67M, 370M, and 1B parameters on FineWeb-Edu and adds a tuned Muon reference point.
\Cref{sec:diag} probes the mechanism on the 67M run, and \Cref{sec:main-deltaSGD} checks generality with AK-SGD and AK-AdamW on CIFAR-10. 
Appendix~\ref{app:setup} gives the full setup for every run.

\subsection{Main Results: Validation Loss on FineWeb-Edu}
\label{sec:main}

\begin{figure}[!tbp]
  \centering
  \begin{subfigure}[t]{0.325\linewidth}
    \includegraphics[width=\linewidth]{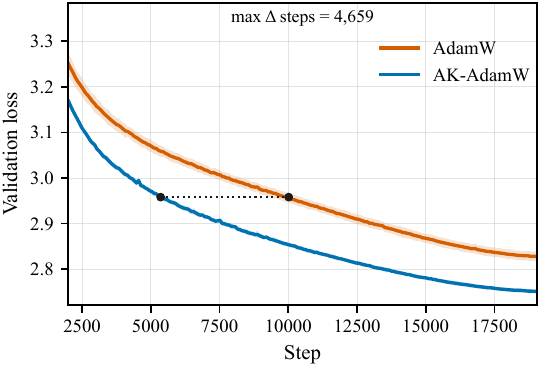}
    \caption{67M, 3 seeds, 10BT.}
    \label{fig:loss-samples-67m}
  \end{subfigure}\hfill
  \begin{subfigure}[t]{0.325\linewidth}
    \includegraphics[width=\linewidth]{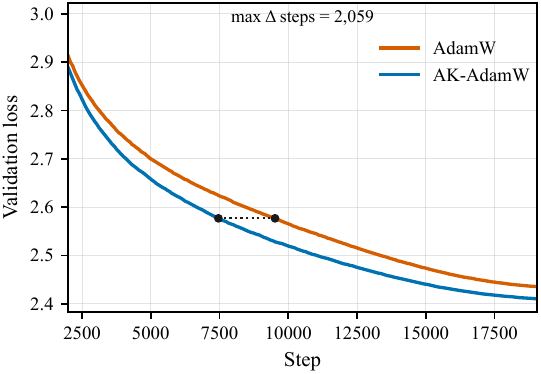}
    \caption{370M, 3 seeds, 10BT.}
    \label{fig:loss-samples-370m}
  \end{subfigure}\hfill
  \begin{subfigure}[t]{0.325\linewidth}
    \includegraphics[width=\linewidth]{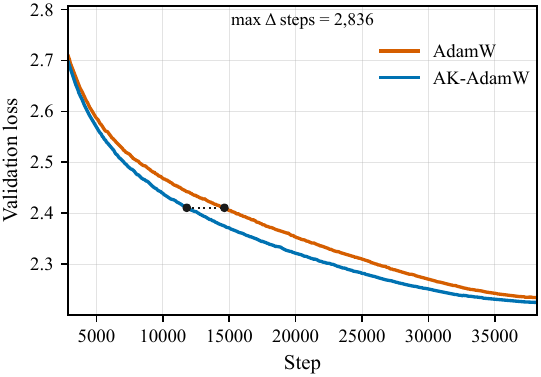}
    \caption{1B, 1 seed, 20BT.}
    \label{fig:loss-samples-1b}
  \end{subfigure}
  \caption{Validation loss against optimizer step on FineWeb-Edu. Curves at 67M and 370M are seed means with per-seed bands, and the 1B curve, on a Chinchilla-optimal budget (Appendix~\ref{app:setup}), is a single seed. The annotation marks the largest step saving over matched validation-loss levels, measured on the seed-mean curve, and curves are cropped at the left to omit the early high-loss region.}
  \label{fig:main-loss}
\end{figure}

\textbf{Per-sample step efficiency.}
At all three scales AK-AdamW reaches AdamW's validation loss in substantially fewer steps and continues to a lower terminal value (\Cref{fig:main-loss}). 
The curves separate shortly after warmup and do not cross again, so AK-AdamW is below AdamW at every matched step over the entire training horizon.
\Cref{tab:scale} reports the paired gap against AdamW across matched validation-loss levels. The sign of the paired gap is the same in every run at every scale, and at 370M the mean gap is $13.7$ times its seed standard deviation. Every hyperparameter at 370M and 1B is carried over from the 67M sweep with no re-tuning, under the $\mu$P rule of \Cref{thm:mup-eta}. The two branches share initialization, data order, and schedule. The advantage is evident early and persists through the cosine decay, as \Cref{thm:tracking} and \Cref{prop:nc1-main} predict for a drifting target without assuming convexity or stationarity (Appendix~\ref{app:scope}). 

\begin{table}[!ht]
  \centering
  \caption{AK-AdamW against AdamW on FineWeb-Edu across matched validation-loss levels, taken over the range plotted in \Cref{fig:main-loss}. 
  Loss gap is in nats and step reduction is the fraction of AdamW steps saved to reach the same loss, each reported as mean and maximum over the matched levels with the seed standard deviation. 
  The 1B run is a single seed on the Chinchilla-optimal budget Appendix~\ref{app:setup}.}
  \label{tab:scale}
  \small
  \begin{tabular}{lccccc}
    \toprule
    Scale (seeds, tokens) & Mean loss gap & Max loss gap & Mean step red. & Max step red. \\
    \midrule
    67M \ \ (3 seeds, 10BT)  & $0.0936 \pm 0.0150$ & $0.1077 \pm 0.0146$ & $40.85 \pm 4.99\%$ & $46.64 \pm 4.52\%$ \\
    370M (3 seeds, 10BT) & $0.0383 \pm 0.0028$ & $0.0501 \pm 0.0016$ & $17.79 \pm 0.99\%$ & $22.14 \pm 0.81\%$ \\
    1B \ \ \ (1 seed, 20BT)  & $0.0245$ & $0.0361$ & $13.47\%$ & $19.37\%$ \\
    \bottomrule
  \end{tabular}
\end{table}

\textbf{Per-FLOP compute efficiency.}
\Cref{tab:cost} collects the cost side of the comparison, and Appendix~\ref{app:realized-flops} gives the counted overhead and the measured step time at all three scales, with the wall-clock curves in \Cref{fig:wallclock}. Unlike the per-step comparison, each cost axis has a crossing point. At 67M both crossings come before step $2{,}000$, where the matched levels begin. At 370M the FLOP crossing is at $4.50$ EFLOPs of the $27.1$ EFLOP horizon, about a sixth of the way through the run, and the wall-clock crossing comes later, near $7.0$ hours, because step time sits above the FLOP count. Past its crossing on each axis AK-AdamW stays ahead at both scales.

\begin{table}[!ht]
  \centering
  \caption{Cost of the deployed algorithm. Overhead is the per-step FLOP overhead counted by our training suite, and step time is measured in the training loop on a single H200 with the same microbatch size in both arms. Each savings column averages the matched validation-loss levels of the seed-mean curves past the crossing point on its own axis, with the best level in parentheses, at the two scales with a full matched-loss measurement (Appendix~\ref{app:flops-memory}).}
  \label{tab:cost}
  \small
  \begin{tabular}{lccccc}
    \toprule
    Scale & FLOP overhead & Step time & FLOPs saved & Wall clock saved \\
    \midrule
    67M  & $10.0\%$ & $1.1526\times$ & $35.1\%$ ($41.1\%$) & $32.0\%$ ($38.3\%$) \\
    370M & $15.5\%$ & $1.1775\times$ & $\phantom{0}5.6\%$ ($\phantom{0}9.4\%$) & $\phantom{0}4.0\%$ ($\phantom{0}7.7\%$) \\
    \bottomrule
  \end{tabular}
\end{table}

\textbf{A reference point comparison among structured optimizers.}
Muon addresses anisotropy elsewhere, by post-processing the buffer rather than changing how it is formed. We add a Muon baseline tuned under the same Optuna protocol and run with 3 seeds at 67M and 370M. AK-AdamW sits below it at both scales, with a mean loss gap of $0.120 \pm 0.008$ nats at 67M and $0.109 \pm 0.003$ at 370M and the same sign in every run. Muon received the largest per-dimension search budget and was tuned under its own $\mu$P rule \citep{qiu2025hptransfer}, and we report it as a baseline tuned inside our protocol rather than a general claim about Muon (Appendix~\ref{app:structured}).

\subsection{Mechanistic Diagnostics: Validating the Theoretical Picture}
\label{sec:diag}

Section~\ref{sec:methodology} predicts that three quantities differ between EMA momentum and \DeltaMomentum, at three points of one mechanism. 
They are the buffer's alignment with the population gradient (\Cref{prop:first-order-optimal}), 
its function-space prediction error (\Cref{def:fms}), 
and the conditioning of the input-feature covariance (\Cref{thm:fixedpoint}). 
We probe each at fixed evaluation steps on the 67M run with a per-layer diagnostic harness on a held-out probe batch.

\textbf{Improved gradient-estimator quality.}
\label{sec:diag-grad}
By \Cref{prop:first-order-optimal}, the buffer's alignment with the population gradient governs per-step loss decrease.
We measure the momentum--gradient cosine
and the relative prediction error
between the pre-step buffer $M_t$ and current gradient $g_t$, on global parameters 
(Appendix~\ref{app:mom-grad-cos-sim}).
Figure~\ref{fig:diag-grad-func}(a,b) shows that AK-AdamW keeps a \emph{strictly higher} cosine than AdamW at every logged step and a lower relative prediction error $e_t$ at most logged steps, with the largest values of $e_t$ belonging to AdamW,
consistent with the gradient-estimator interpretation of \Cref{sec:gradient-estimator}.

\begin{figure}[!tbp]
  \centering
  \newcommand{\subfigheight}{3.0cm}
  \captionsetup[subfigure]{justification=centering,singlelinecheck=false}

  \begin{subfigure}[t]{0.32\linewidth}
    \centering
    \includegraphics[width=\linewidth,height=\subfigheight,keepaspectratio]{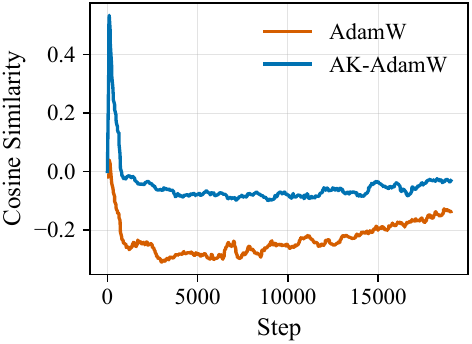}
    \caption{Momentum--gradient cosine similarity.}
    \label{fig:cossim}
  \end{subfigure}\hfill
  \begin{subfigure}[t]{0.32\linewidth}
    \centering
    \includegraphics[width=\linewidth,height=\subfigheight,keepaspectratio]{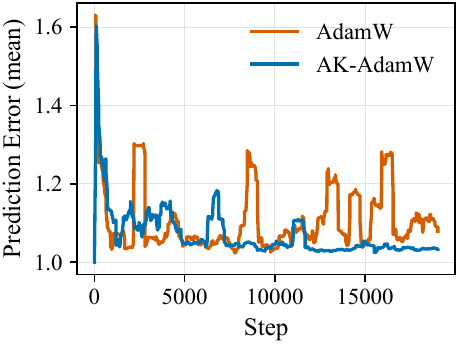}
    \caption{Momentum--gradient prediction error.}
    \label{fig:grad-error}
  \end{subfigure}\hfill
  \begin{subfigure}[t]{0.32\linewidth}
    \centering
    \includegraphics[width=\linewidth,height=\subfigheight,keepaspectratio]{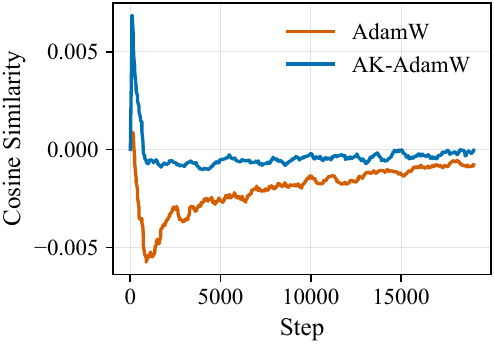}
    \caption{Gradient--prediction cosine similarity.}
    \label{fig:dy-cossim}
  \end{subfigure}

  \caption{Mechanistic diagnostics on momentum buffer on the 67M run.
    AK-AdamW (blue) maintains \textbf{(a)}~higher 
    $\cos(M_t, g_t)$
    and \textbf{(c)}~higher 
    $\cos(M_t \hat{x},\, \delta)$ 
    than AdamW (orange) at every logged step, and \textbf{(b)}~lower $e_t = \|M_t - g_t\|_F / \|g_t\|_F$ at most logged steps. 
    Negative cosines come from stochastic descent itself, not the buffer rule, and occur for both optimizers (Appendix~\ref{app:mom-grad-cos-sim}). The diagnostic quantity is the gap between the two curves.
    }
  \label{fig:diag-grad-func}
\end{figure}

\textbf{Function-space prediction error.}
\label{sec:diag-func}
By \Cref{def:fms}, \DeltaMomentum runs online stochastic gradient descent on the prediction loss \(\frac{1}{2}\|\delta - M\hat{x}\|^2\), a loss that EMA does not optimize at all, so its buffer should predict the output-side error \(\delta\) from the input \(\hat{x}\) better.
At each checkpoint we freeze \(M_t\) for both optimizers, predict \(\widehat{\delta}_t = M_t \hat{x}_t\) per layer on a held-out probe batch, and report $\cos(\hat\delta, \delta)$ in \Cref{fig:dy-cossim}.

\textbf{Robustness against anisotropic collapse of input features.}
\label{sec:diag-feat}
\Cref{thm:fixedpoint} scales the steady-state buffer along an eigendirection $u_i$ of the input-feature covariance by $\phi_i = 1/(\mu+\lambda_i)$ rather than by 1 as in EMA, which counteracts the dominance of high-density directions. If this operates in training, layers under AK-AdamW should keep a better-conditioned input-feature covariance than under AdamW. \Cref{fig:diag-feat} reports the layerwise means of the condition number $\lambda_{\max}/\lambda_{\min}$ and of the effective rank $\exp(H(\boldsymbol\lambda/\sum_i\lambda_i))$, with $H$ the Shannon entropy, and the maximum top-1 variance fraction $\lambda_{\max}/\sum_i\lambda_i$. 
Past $3{,}000$ steps, AK-AdamW keeps the condition number and the top-1 fraction lower, and it keeps the effective rank higher until the two meet near step $15{,}000$.
The diagonal second moment of AdamW registers only the diagonal of $\Sigma_x$, never its eigendirections, so these statistics separate the delta correction from it (Appendix~\ref{app:second-moment}). 

\textbf{Interpretation. }
All three diagnostics move in the direction \Cref{sec:methodology} predicts.

\begin{figure}[!tbp]
  \centering
  \newcommand{\subfigheight}{2.65cm}
  \begin{subfigure}[t]{0.32\linewidth}
    \includegraphics[width=\linewidth]{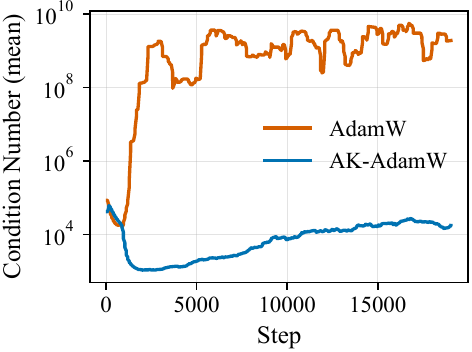}
    \caption{Condition number (mean).}
    \label{fig:cond}
  \end{subfigure}\hfill
  \begin{subfigure}[t]{0.32\linewidth}
    \includegraphics[width=\linewidth]{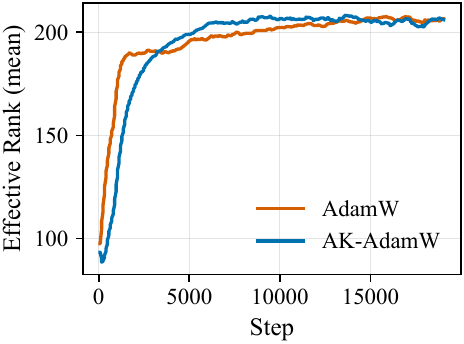}
    \caption{Effective rank (layerwise mean).}
    \label{fig:effrank}
  \end{subfigure}\hfill
  \begin{subfigure}[t]{0.32\linewidth}
    \includegraphics[width=\linewidth]{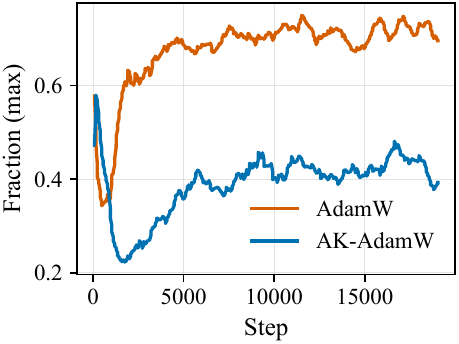}
    \caption{Top-1 variance fraction (max).}
    \label{fig:top1}
  \end{subfigure}
  \caption{Input-feature covariance diagnostics on the 67M run.
    Past $3{,}000$ steps, layers trained with AK-AdamW (blue) have a lower condition number and a lower top-1 fraction than layers trained with AdamW (orange), and a higher effective rank until the two meet near step $15{,}000$, consistent with the implicit input-side preconditioning predicted by \cref{thm:fixedpoint}.}
  \label{fig:diag-feat}
\end{figure}

\subsection{Generality Across Optimizer Families and Architectures}
\label{sec:main-deltaSGD}

To verify that the gain transfers beyond Adam's diagonal preconditioner, we compare AK-SGD with SGD-momentum on a 3-hidden-layer MLP on CIFAR-10. AK-SGD reaches SGD's validation loss earlier and converges to a strictly lower terminal value, with the gap opening right after the first epoch and persisting throughout. Validation curves and test accuracy are in Appendix~\ref{app:cifar-10}.

To verify that it also transfers beyond language modeling and beyond a single architecture, we run ResNet-18 and ViT-Tiny on CIFAR-10 for 120 epochs, AK-AdamW against AdamW with 3 seeds and an identical log-scale learning-rate sweep for each optimizer. \Cref{tab:vision} reports the training-loss gap at matched epochs, averaged over the phase where the curves separate, and the sign of the gap is the same in all runs. Appendix~\ref{app:cifar-vision} gives the curves, the averaging window, and why each architecture is read on the curve that measures optimization rather than the inductive bias of the model. We read this as a generality check rather than a second pillar of evidence, and language-model pretraining remains where the claims of this paper are made.

\begin{table}[!ht]
  \centering
  \caption{AK-AdamW against AdamW on CIFAR-10 over 120 epochs with 3 seeds. Gaps are training loss at matched epochs, in nats and relative to the AdamW loss at that epoch, averaged over the phase where the curves separate. We report training loss for both architectures here so the two rows are on the same footing, and Appendix~\ref{app:cifar-vision} gives the validation-loss curve for ResNet-18. Curves are in Appendix~\ref{app:cifar-10}.}
  \label{tab:vision}
  \small
  \begin{tabular}{lcccc}
    \toprule
    Architecture & Mean gap (nats) & Mean gap (\%) & Max gap (nats) & Max gap (\%) \\
    \midrule
    ViT-Tiny   & $0.079 \pm 0.013$ & $25.1 \pm 2.7$ & $0.140 \pm 0.017$ & $45.5 \pm 3.5$ \\
    ResNet-18  & $0.041 \pm 0.006$ & $11.1 \pm 1.6$ & $0.081 \pm 0.016$ & $15.1 \pm 1.8$ \\
    \bottomrule
  \end{tabular}
\end{table}

\section{Conclusion}
\label{sec:conclusion}

\DeltaMomentum replaces the EMA accumulator with the classical delta rule, treating the first-moment buffer as an online linear associative memory of the factored gradient $\delta_t x_t^\top$. The anisotropic memory transport mechanism (Lemma~\ref{lem:transport}) unifies its steady-state, per-direction convergence, and non-stationary behavior.
AK-AdamW reduces steps-to-target loss by up to $46.64 \pm 4.52\%$ at 67M and $22.14 \pm 0.81\%$ at 370M on FineWeb-Edu, the gain holds at 1B on a Chinchilla-optimal budget, and AK-SGD, ResNet-18, and ViT-Tiny show the same pattern on CIFAR-10.

\textbf{Limitations and Future Work.}
Experiments cover language modeling up to 1B parameters and CIFAR-10 classifiers. Larger scales and domains such as image generation and reinforcement learning remain open.
The counted per-step cost rises with width to $18.1\%$ at 1B, where measured step time follows it, and the gap between them is larger at smaller widths.
Our two modalities, language and images, are trained separately, and a run that interleaves them is the natural next step (Appendix~\ref{app:multimodal}).
Buffer-level compositions with Muon, Shampoo, and SOAP, and a controlled comparison against a gradient-keyed delta rule on the buffer \citep{behrouz2026nested}, which isolates the choice of key, are left to future work.

\begin{ack}
This work is supported by NSF Grants 2339112, 2512805, Jane Street, and Pennsylvania Infrastructure Technology Alliance.
\end{ack}

\bibliographystyle{plainnat}
\bibliography{references}

\clearpage
\appendix

\section{Symbols and Notation}
\label{app:symbols-and-notations}

\begin{table}[!ht]
\centering
\caption{Notation used in the analysis. 
Hatted symbols belong to the normalized-key reading \eqref{eq:dm-norm-update}.
}
\label{tab:notation}
\small
\setlength{\tabcolsep}{4pt}
\begin{tabular}{@{}p{0.30\linewidth}p{0.40\linewidth}p{0.26\linewidth}@{}}
\toprule
Symbol & Meaning & Where it enters \\
\midrule
$x,\ \delta$ & input activation, output-side error $\partial \mathcal{L}/\partial y$ & key and value of the association \\
$\xhat = x/\|x\|_2$ & normalized key & deployed update \\
$X, \Xhat, \Delta$ & batched inputs, normalized inputs, errors & Appendix~\ref{app:flops-memory} \\
$\Sigx = \E[xx^\top]$ & input second moment & unnormalized analysis \\
$\Sighat = \E[\xhat\xhat^\top]$ & key second moment, $\tr\Sighat = 1$ & deployed analysis \\
$\lambda_i,\ \lamhat_i,\ u_i$ & eigenvalues of $\Sigx$, of $\Sighat$, shared eigenvectors & per-direction statements \\
$\kappa_i = u_i^\top \E[\|x\|\,\xhat\xhat^\top] u_i$ & per-direction gradient coupling & trace of \eqref{eq:iter-matrices} \\[2pt]
$\Gbar = \E[\delta x^\top]$, $\Ghat = \E[\delta \xhat^\top]$ & mean gradient in raw and normalized keys & fixed points \\
$G$ & minibatch weight gradient, the batch estimate of $\Gbar$ & deployed update \\
$M,\ \Mstar$ & momentum buffer and its fixed point & \Cref{thm:fixedpoint} \\
$\beta,\ \eta$ & EMA decay, delta-rule coefficient & \eqref{eq:dm-update} \\
$\tbeta_i = \beta - \eta\lamhat_i$ & per-direction memory contraction & \Cref{lem:transport} \\
$\mu = (1-\beta)/\eta$ & Tikhonov ridge & \Cref{thm:fixedpoint} \\
$\tau_i = 1/((1-\beta) + \eta\lamhat_i)$ & per-direction memory horizon & \eqref{eq:horizon} \\
$r_i = \eta\lamhat_i/(1-\beta)$ & forgetting speed-up over EMA & Appendix~\ref{app:beyond-linreg} \\
$m,\ n,\ N,\ \gamma$ & output width, input width, tokens per step, gated-MLP ratio & Appendix~\ref{app:flops-memory} \\[2pt]
\bottomrule
\end{tabular}
\end{table}

\section{Extended Related Work}
\label{app:related}

\textbf{Momentum and acceleration.}
First-moment EMA momentum has been a near-universal component of deep-network optimizers since heavy-ball momentum \citep{polyak1964some}, which is the source of the view of the buffer as inertia. A separate line starting from Nesterov's accelerated gradient \citep{nesterov1983method} builds accelerated methods on a lookahead rather than on the EMA first moment used in SGD-momentum and Adam-style optimizers, so we place it here rather than on the buffer itself. The empirical importance of momentum in neural-network training was established by \citet{sutskever2013importance}, together with its interaction with initialization and with momentum schedules. The reading of the buffer as gradient denoising is developed directly in the smoothing and variance-reduction line \citep{cutkosky2019momentum, xu2023momentum}. 
The isotropic, scalar-decay form of EMA has been retained in essentially every adaptive optimizer derived from Adam \citep{kingma2014adam, loshchilov2017decoupled}.
\DeltaMomentum departs from this lineage by replacing the EMA update itself with a delta-rule update on a function-space prediction loss, keyed on the input activation. 
The delta-rule constructions of Behrouz et al. \citep{behrouz2026nested} depart from the same lineage, on a loss keyed on the gradient instead.

\textbf{Online associative memory and the delta rule.}
The delta rule \citep{widrow1988adaptive} is the canonical online stochastic gradient method for linear least squares with key--value targets. 
In recent deep-learning literature, the delta rule has reappeared as the backbone of linear-attention fast-weight programmers \citep{schlag2021linear, yang2024parallelizing}, where a memory matrix stores key--value associations across a sequence and is updated online to minimize retrieval error. 
A linear layer trained by gradient descent is itself such a memory, written additively, since its weight is the initialization plus a sum of outer products of its inputs and errors and its output is an unnormalized attention read over that history \citep{aizerman1964,irie2022dual}.
\DeltaMomentum imports this online-associative-memory machinery into the optimizer, so the first-moment buffer plays the role of the memory matrix, the layer's input activations and output-side errors play the roles of keys and values, and the per-step update is exactly the regularized delta rule. 
The view of the optimizer and its momentum as an associative memory over gradients has also appeared in Behrouz et al. \citep{behrouz2025s, behrouz2026nested}. 
They factor the layer gradient into an outer product of the input activation and a local error and carry that reading to the momentum term, writing the momentum as an associative memory whose objective is a free design choice rather than a fixed one. 
They then instantiate that choice with a delta rule keyed on the gradient. 
They also give a delta rule that does key on the input activation, but it replaces the descent step on the weights rather than the momentum buffer. 
This paper newly devises the activation-keyed instantiation of the momentum objective and the theory that it supports. 
What the choice of objective decides is the form of the resulting rule. 
The dot-product objective $\langle M x, \delta \rangle$ has gradient step $\delta x^\top$, so it returns the Hebbian accumulator of \Cref{subsec:delta-rule}, which is EMA momentum with its single decay rate. 
The function-space prediction loss of \Cref{def:fms}, on the same key and value, produces the state-dependent forgetting factor of \eqref{eq:dm-factored} instead. 
\Cref{tab:delta-rule-placement} places the three rules side by side. 

\begin{table}[h]
\centering
\caption{Where the delta rule is applied and what keys it. The first two rows are the constructions of \citet{behrouz2026nested}, written throughout in the notation of \Cref{tab:notation}, so $\beta$ and $\eta$ denote the decay and the step size belonging to the rule in that row rather than a value shared across rows. The forgetting factor right-multiplies the object in the second column. Of the two buffer-level rules, only the activation-keyed one has a forgetting factor whose expectation is $\beta I - \eta \Sigx$, the matrix \Cref{lem:transport} and everything downstream of it act on.}
\label{tab:delta-rule-placement}
\small
\begin{tabular}{lllll}
\toprule
Rule & Applied to & Key & Value & Forgetting factor \\
\midrule
Weight-level \citep{behrouz2026nested} & weights & input activation $x$ & output-side error $\delta$ & $I - \eta\, x x^\top$ \\
Buffer-level \citep{behrouz2026nested} & momentum & gradient $g$ & preconditioner & $\beta I - g^\top g$ \\
\DeltaMomentum & momentum & normalized input activation $\xhat$ & output-side error $\delta$ & $\beta I - \eta\, \xhat \xhat^\top$ \\
\bottomrule
\end{tabular}
\end{table}

\textbf{Anisotropy in deep-network optimization.}
The anisotropy of the deep-network loss landscape has been documented extensively. 
Hessian-eigenvalue-density studies \citep{sagun2017empirical, ghorbani2019investigation, papyan2019measurements} report a small bulk of large eigenvalues followed by a long, near-flat tail spanning several orders of magnitude. Activation-anisotropy studies in NLP \citep{mu2017all, ethayarajh2019contextual} show that learned representations concentrate disproportionately along a handful of dominant
directions. 
These findings have motivated optimizers that explicitly account for anisotropy. K-FAC \citep{martens2015optimizing} approximates the Fisher information by Kronecker-factored input/output curvature. Shampoo \citep{gupta2018shampoo} applies Kronecker-factored preconditioning to the gradient. SOAP \citep{vyas2024soap} refines the eigenbasis used by Shampoo. Muon \citep{jordan2024muon} applies Newton--Schulz orthogonalization to the EMA momentum to flatten its singular-value spectrum. Adafactor
\citep{shazeer2018adafactor} factorizes the second-moment statistics. 
A unifying property of these methods is that they retain the EMA momentum update and apply external structural processing (preconditioning before, post-conditioning after) to address anisotropy. 
\DeltaMomentum addresses anisotropy at a different level of the optimizer stack, inside the momentum update itself. 
Therefore, \DeltaMomentum is complementary to, rather than in competition with, these methods.

\textbf{Activation-side preconditioning.}
FOOF \citep{benzing2022gradient} preconditions the gradient by the inverse second moment of the input activations alone and derives this as gradient descent on neurons.
Concurrent with this work, DoPr \citep{zhang2026dopr} pairs a gradient-based preconditioner such as Adam or Muon with an activation-based preconditioner in the style of K-FAC, on the argument that non-isotropic input activations degrade feature learning. 
The object it adds is the input-side Kronecker factor, chosen over the full K-FAC preconditioner to avoid the output-side factor and the extra optimizer state it needs. That is the factor \Cref{thm:fixedpoint} shows \DeltaMomentum reaches implicitly, without inversion, damping, or state beyond the momentum the base optimizer already keeps. 
Newton-Muon \citep{du2026newtonmuon} reaches the same factor from a different direction. It approximates the layer loss by a quadratic surrogate in the gradient, an output-side curvature matrix, and the matrix of input activations, and minimizing that surrogate right-preconditions the gradient by the inverse input second moment before the usual Muon pipeline. It reads standard Muon as a Newton-type method that drops this right preconditioner. 
Both methods differ from ours in level, since each forms an explicit preconditioner and applies it to a buffer that is still an EMA, while \DeltaMomentum changes the buffer itself, so they compose with it rather than substitute for it. 
DoPr reports downstream test-time performance under rollout and Newton-Muon reports iterations to a target validation loss on a GPT-2 speedrun configuration, so neither set of results is directly comparable with ours.

\textbf{AK-AdamW vs.\ AdamW is a controlled head-to-head that isolates the first-moment update.}
The empirical contribution of this paper is the delta-rule update applied to the first-moment buffer, and cleanly isolating that contribution requires a baseline that differs from AK-AdamW \emph{only} in the first-moment update. AdamW is the unique such baseline. The second moment, bias correction, and decoupled weight decay are identical between AK-AdamW and AdamW, and setting the state-dependent forgetting factor of \eqref{eq:dm-factored} to zero returns \eqref{eq:ema-update} up to the constant factor $\eta/(1-\beta)$, so AdamW is the null case of the mechanism and any difference in the loss curve is attributable to the EMA-vs-delta swap. Any other delta-rule buffer, including one keyed on the gradient, differs from AK-AdamW in the choice of key rather than in whether the mechanism is present, so it isolates a different variable. 
Head-to-heads against Muon, Shampoo, or SOAP would conflate the first-moment contribution with Newton--Schulz orthogonalization or with Kronecker-factored preconditioning applied at a different level of the optimizer stack, which is precisely the level at which those methods address anisotropy and at which \DeltaMomentum is complementary, not substitutive. The natural compositions, \textsc{AK-Muon}, \textsc{AK-Shampoo}, and \textsc{AK-SOAP}, inherit \DeltaMomentum's first-moment update while retaining the structural processing of the base optimizer, and we identify them as future work in \Cref{sec:conclusion}, and an isolated AK-AdamW vs.\ AdamW comparison is the prerequisite for those follow-ups, not a substitute for them. The AK-SGD vs.\ SGD-momentum result on CIFAR-10 (\Cref{sec:main-deltaSGD}) provides a parallel controlled comparison in the family without an adaptive denominator, verifying that the gain is attributable to the delta-rule mechanism rather than to its interaction with Adam's diagonal preconditioner.

\section{Stability of the Normalized-key \DeltaMomentum}
\label{app:stability_norm_key_dm}
The normalized-key variant 
gives input-scale-independent contraction under a
width-independent condition \(\eta<\beta\), unlike the unnormalized-key variant.
This is the main practical reason for
deploying it in deep-network training.

\begin{proposition}[Bounded directional second moment]
\label{prop:bounded-sigma-hat}
Let $x \in \R^n$ be any random vector with $\Pr(\|x\|>0) = 1$, and let
$\xhat = x/\|x\|$. Then $\Sighat := \E[\xhat\xhat^\top]$ satisfies
\begin{equation}
\tr(\Sighat) = 1 \qquad\text{and}\qquad 0 \le \lamhat_i \le 1 \text{ for all } i,
\label{eq:sigma-hat-bound}
\end{equation}
with $\sum_i \lamhat_i = 1$.
\end{proposition}

\begin{proof}
$\tr(\Sighat) = \E[\tr(\xhat\xhat^\top)] = \E[\xhat^\top\xhat] = \E[\|\xhat\|^2] = 1$.
Each eigenvalue is non-negative since $\Sighat$ is positive semi-definite,
and bounded above by the trace.
\end{proof}

\textbf{Normalized-key variant. } 
$\xhat_t := x_t / \|x_t\|$, giving
\begin{equation}
M_t = \beta M_{t-1} + \eta(\delta_t - M_{t-1}\xhat_t) \xhat_t^\top.
\label{eq:dm-norm-update}
\end{equation}

\begin{theorem}[Input-scale-independent contraction of the normalized variant]
\label{thm:unconditional-stability}
The normalized-key \DeltaMomentum recurrence~\eqref{eq:dm-norm-update} 
contracts in expectation along
every input direction whenever \(\eta < \beta\),
\emph{independent of} input scale, layer index, or training step. In
contrast, the original unnormalized-key \DeltaMomentum~\eqref{eq:dm-update} requires the
input-dependent and time-varying condition $\eta < \beta/\lambda_{\max}(\Sigx)$.
\end{theorem}

\begin{proof}
By Lemma~\ref{lem:transport}, the expected per-direction memory contraction along the eigendirection $\uhat_i$ of $\Sighat$ is
$\tbeta_i = \beta - \eta\lamhat_i$. By
Proposition~\ref{prop:bounded-sigma-hat}, $\lamhat_i \le 1$, hence
$\tbeta_i \ge \beta - \eta > 0$ under the stated condition, so the buffer
contracts along every direction in expectation. For the unnormalized
variant the contraction along eigenvector $u_i$ of $\Sigx$ is
$\beta - \eta\lambda_i(\Sigx)$, which requires
$\eta < \beta/\lambda_{\max}(\Sigx)$ and depends on input scale.
\end{proof}

The decoupling of $\eta$ from input scale is the main practical reason
for using the normalized variant in deep-network training, where activation
scales vary by orders of magnitude across layers and through training.

\begin{remark}[Why normalization]
The forgetting operator is $\beta I - \eta\, \xhat_t \xhat_t^\top$, and its
rank-one part has spectral norm $\|\xhat_t \xhat_t^\top\|_{\mathrm{op}} = 1$ at
every width. This decouples $\eta$ from width, since stability requires only
$\eta < 1 + \beta$ and $\eta = \bigO(1)$ is admissible
across all model widths. Without normalization, $\|x_tx_t^\top\|_{\mathrm{op}} \sim n$,
forcing $\eta = \bigO(1/n)$. This is the foundation of $\mu$P-clean
width transfer for \DeltaMomentum.
\end{remark}

\section{Why the Momentum Buffer Is Necessarily a Gradient Estimator}
\label{app:taylor}

\begin{proposition}[First-order optimal momentum]
\label{prop:first-order-optimal}
Let $\mathcal{L}$ be three times continuously differentiable near $\theta_t$ with Hessian $H_t$. For any
choice of $M_t$, the update $\theta_{t+1} = \theta_t - \alpha M_t$ satisfies
\begin{equation}
\mathcal{L}(\theta_{t+1})
= \mathcal{L}(\theta_t)
  - \alpha\,\langle M_t,\Gbar(\theta_t)\rangle_F
  + \frac{\alpha^2}{2}\,\vecop(M_t)^\top H_t\,\vecop(M_t)
  + O(\alpha^3),
\label{eq:taylor}
\end{equation}
where $\Gbar(\theta_t) = \nabla_\theta \mathcal{L}(\theta_t)$. To first order in
$\alpha$, the loss decrease $\mathcal{L}(\theta_t) - \mathcal{L}(\theta_{t+1})$ is maximized over
$M_t$ of fixed Frobenius norm $\|M_t\|_F$ at $M_t \propto \Gbar(\theta_t)$.
\end{proposition}

\begin{proof}
Equation~\eqref{eq:taylor} is the second-order Taylor expansion of $\mathcal{L}$ at
$\theta_t$ with displacement $\Delta\theta = -\alpha M_t$. It holds pointwise
for any realization of $M_t$, and no probabilistic averaging is required. The
first-order term $-\alpha\langle M_t,\Gbar(\theta_t)\rangle_F$ is linear in
$M_t$, and by Cauchy--Schwarz $\langle M_t,\Gbar_t\rangle_F$ is maximized
over $M_t$ of fixed Frobenius norm by alignment $M_t \propto \Gbar_t$.
\end{proof}

Proposition~\ref{prop:first-order-optimal} is a Taylor identity, not a probabilistic
assumption, since any realized buffer $M_t$ that deviates from $\Gbar(\theta_t)$
pays a first-order shortfall
\begin{equation}
\Delta_{\mathrm{shortfall}}
= \alpha\,\langle \Gbar_t - M_t,\,\Gbar_t\rangle_F
\;\ge\; 0
\quad\text{(for fixed $\|M_t\|_F = \|\Gbar_t\|_F$)},
\end{equation}
in the realized loss decrease relative to the optimal direction
$M_t = \Gbar_t$. The quality of $M_t$ as an estimator of $\Gbar_t$ therefore
directly governs optimization efficiency. What differentiates between EMA
and \DeltaMomentum is the loss function on which the buffer performs online
stochastic gradient descent.

\section{Proofs for Estimation Objectives and Fixed Point Analysis}
\label{app:estimation-proofs}

\subsection{Divergent optima}
\label{app:divergent-optima}

\begin{proof}
\emph{Weight-space.} 

$J_w(M) = \tfrac{1}{2}\|M\|_F^2 - \langle M, \Gbar\rangle_F + \tfrac{1}{2}\E[\|\delta x^\top\|_F^2]$,
$\nabla_M J_w = M - \Gbar = 0 \Rightarrow M_w^{*} = \Gbar$.

\emph{Function-space.} 

$J_f(M) = \tfrac{1}{2}\E[\|\delta\|^2] - \tr(M^\top \Gbar) + \tfrac{1}{2}\tr(M\Sigx M^\top)$,
where $\E[\delta^\top M x] = \tr(M^\top \Gbar)$ and $\E[x^\top M^\top M x] = \tr(M\Sigx M^\top)$,
$\nabla_M J_f = -\Gbar + M\Sigx = 0 \Rightarrow M_f^{*} = \Gbar\,\Sigx^{-1}$.
\end{proof}

\subsection{Update rules as online SGD}
\label{app:updates-as-sgd}
\begin{proof}
\emph{(a)} $\nabla_M \ell_t^w = M - g_t$, so one SGD step with step $\alpha'$ gives
$M \leftarrow M - \alpha'(M - g_t) = (1-\alpha') M + \alpha' g_t$. Setting
$\alpha' = 1-\beta$ recovers EMA.
\emph{(b)} $\nabla_M \ell_t^f = -(\delta_t - M x_t) x_t^\top$, so one SGD step with
step $\eta$ gives $M \leftarrow M + \eta(\delta_t - Mx_t) x_t^\top$. At stationarity,
$\E[\nabla_M \ell_t^f] = -(\Gbar - M\Sigx) = 0 \Rightarrow M^{*} = \Gbar\,\Sigx^{-1}$.
\end{proof}

\subsection{Fixed Point Convergence of \DeltaMomentum}
\label{app:fixed-point}
\begin{proof}
Take expectations of~\eqref{eq:dm-factored} conditional on $M_{t-1}$:
$\E[M_t \mid M_{t-1}] = M_{t-1}(\beta I - \eta\Sigx) + \eta\Gbar$.
At stationarity $M^{*} = M^{*}(\beta I - \eta\Sigx) + \eta\Gbar$, hence
$M^{*}[(1-\beta)I + \eta\Sigx] = \eta\Gbar$ and~\eqref{eq:dm-fixedpoint}.
Projection onto $u_i$ gives~\eqref{eq:dm-eigen}.
\end{proof}

\subsection{Memory-theoretic reading} 
\label{app:memory-theoretic-reading}
Theorem~\ref{thm:fixedpoint} can be
read directly off Lemma~\ref{lem:transport}: along $u_i$, the steady-state
ratio of injection to forgetting is
$\eta / [(1-\beta) + \eta\lambda_i] = \eta\tau_i$, the product of per-query
injection gain $\eta$ and direction-dependent memory horizon $\tau_i$.
Directions queried more often have shorter memory horizons, which reduces
the steady-state retained mass along those directions, and this down-weighting
\emph{is} input-side whitening.

\begin{remark}[K-FAC connection]
\label{rem:kfac-app}
The K-FAC approximation of the Fisher information for a linear layer is
$F \approx \E[\delta\delta^\top]\otimes \E[x x^\top]$, with natural-gradient
update $\E[\delta\delta^\top]^{-1}\, \Gbar\,\Sigx^{-1}$. The pure delta-rule
fixed point $\Gbar\,\Sigx^{-1}$ is exactly the input-side factor of this
Kronecker product. \DeltaMomentum thus performs implicit input-side
natural-gradient preconditioning without explicitly computing or inverting
$\Sigx$, and the regularization parameter $\mu$ provides a numerically stable
interpolation toward this preconditioner. As in \Cref{rem:kfac}, the deployed
algorithm whitens $\Sighat$ rather than $\Sigx$, so this reading holds up to the
scalar identified in \Cref{cor:fixedpoint-normalized} rather than as an identity.
\end{remark}

\subsection{Estimator and Preconditioner as One Object}
\label{app:estimator-and-preconditioner}

\DeltaMomentum is primarily a gradient estimator, and the preconditioning reading is what that estimator looks like when it is written in weight coordinates. Neither reading is only a statement about the fixed point. What separates \DeltaMomentum from EMA is the space in which the estimate is formed. A layer computes $y = Wx$, so a step $W \leftarrow W - \alpha M$ changes the output on input $x$ by $-\alpha M x$ while the loss asks for $-\alpha\delta$. The job of the buffer is therefore to estimate the map from input to desired output change, which is the regression of $\delta$ on $\xhat$. That is \Cref{def:fms}, minimized by $\Ghat\Sighat^{-1}$, and \eqref{eq:dm-update} is online SGD on it. EMA estimates the raw weight-space gradient instead, which is \Cref{def:wms}, minimized by $\Ghat$. So $\Sighat^{-1}$ is not an operator applied to a gradient after the fact. It is the normal equations of the regression that defines the estimate.

The preconditioning reading follows because weight space and function space are the same linear space under two inner products, $\langle A, B\rangle_F$ and $\E\langle A\xhat, B\xhat\rangle = \tr(A\Sighat B^\top)$. The matrix $\Sighat^{-1}$ converts between them, so writing the function-space estimate in weight coordinates produces the right-preconditioned form of \Cref{thm:fixedpoint}. Substituting the buffer for EMA is appropriate because the two objectives are not in conflict. Both maximize the same linear functional $\E\langle \delta, M\xhat\rangle = \langle M, \Ghat\rangle_F$, and they differ only in which norm is held fixed while doing so, $\|M\|_F$ for \Cref{def:wms} and the induced output change $\E\|M\xhat\|^2$ for \Cref{def:fms}, with the Tikhonov term $\mu$ interpolating between them. The second constraint is the one the role of the buffer calls for, and for a linear layer the identification is exact. In expectation the substitution also cannot reverse a descent direction, since $((1-\beta)I + \eta\Sighat)^{-1} \succ 0$.

\section{Proofs for Linear Regression Dynamics}
\label{app:linreg-proofs}
\label{app:linear-regression}

\subsection{Per-Direction Iteration Matrices}
\label{app:iteration-matrices}

\begin{proof}
The momentum update along $u_i$ is, in expectation,
\begin{align*}
\text{EMA: } & \mu_{t,i} = \beta\mu_{t-1,i} + (1-\beta)\lambda_i\, e_{t,i},\\
\text{Delta: } & \mu_{t,i} = (\beta - \eta\lambda_i)\mu_{t-1,i} + \eta\lambda_i\, e_{t,i},
\end{align*}
where Delta uses the conditional expectation $\E[M_{t-1} x_t x_t^\top \mid M_{t-1}] = M_{t-1}\Sigx$
projected onto $u_i$. The weight update is
$e_{t+1,i} = e_{t,i} - \alpha\mu_{t,i}$. Substituting the momentum recurrence
into the weight recurrence and arranging the state vector as
$(e_{t,i}, \mu_{t-1,i})^\top$ gives~\eqref{eq:iter-matrices} in the unnormalized reading.

For the normalized-key rule \eqref{eq:dm-norm-update} the same derivation gives
\begin{align*}
\text{Delta: } & \mu_{t,i} = (\beta - \eta\lamhat_i)\mu_{t-1,i} + \eta\kappa_i\, e_{t,i},
\end{align*}
because the buffer is contracted by $\Sighat$ while the injected value $\delta_t = E_t x_t$ still carries the raw input, so $\Ghat u_i = \E[\delta\xhat^\top]u_i = e_{t,i}\,\kappa_i$ with $\kappa_i = u_i^\top\E[\|x\|\,\xhat\xhat^\top]u_i$, using $x\xhat^\top = \|x\|\,\xhat\xhat^\top$. This is the form stated in \eqref{eq:iter-matrices}, and it requires that $\Sighat$ and $\E[\|x\|\,\xhat\xhat^\top]$ share eigenvectors. Only the trace moves. The determinant is $\tbeta_i$ either way, so \Cref{thm:detred} and \Cref{thm:strengthened-rate} are unaffected by the split.
\end{proof}

\subsection{Per-direction Determinant Reduction}
\label{app:determinant-reduction}

\begin{proof}
Direct computation. For EMA, $\det A_i^{\mathrm{EMA}} = \beta(1 - \alpha(1-\beta)\kappa_i) - (-\alpha\beta)(1-\beta)\kappa_i
= \beta - \alpha\beta(1-\beta)\kappa_i + \alpha\beta(1-\beta)\kappa_i = \beta$.
For \DeltaMomentum, $\det A_i^{\mathrm{Delta}} = \tbeta_i(1 - \alpha\eta\kappa_i) - (-\alpha\tbeta_i)(\eta\kappa_i)
= \tbeta_i - \alpha\eta\kappa_i\tbeta_i + \alpha\eta\kappa_i\tbeta_i = \tbeta_i$.
The coupling $\kappa_i$ cancels in both, which is why the determinant is set by the memory statistic alone. The inequality $\tbeta_i < \beta$ is immediate from $\eta\lamhat_i > 0$. The unnormalized reading is the same computation with $\kappa_i$ replaced by $\lambda_i$ and $\lamhat_i$ by $\lambda_i$.
\end{proof}

\subsection{Convergence Rate Reduction in \DeltaMomentum's Underdamped Regime}

\begin{lemma}[Determinant lower bound on the spectral radius]
\label{lem:det-bound}
Let $A \in \R^{2\times 2}$ with $\det A > 0$. Then
\begin{equation}
\rho(A) \;\ge\; \sqrt{\det A},
\label{eq:det-bound}
\end{equation}
with equality if and only if the discriminant
$\Delta(A) = \tr(A)^2 - 4\det A$ satisfies $\Delta(A) \le 0$ (i.e., $A$ has
complex-conjugate eigenvalues or a repeated real eigenvalue).
\end{lemma}

\begin{proof}
Let $z_1, z_2$ be the eigenvalues, with $z_1 z_2 = \det A > 0$, so they are
either both real with the same sign or a complex-conjugate pair. By the
AM--GM inequality applied to $|z_1|, |z_2| > 0$,
\[
\rho(A) = \max(|z_1|, |z_2|) \;\ge\; \tfrac12(|z_1|+|z_2|) \;\ge\; \sqrt{|z_1| |z_2|} = \sqrt{\det A}.
\]
Equality in the second inequality holds iff $|z_1| = |z_2|$, which for a
real matrix $A$ holds iff either (a) $z_1 = z_2$ (repeated real eigenvalue,
$\Delta = 0$), or (b) $z_1 = \bar z_2$ is a complex-conjugate pair
($\Delta < 0$).
\end{proof}

\begin{theorem}[Per-direction rate reduction in the underdamped regime]
\label{thm:strengthened-rate}
Assume the stability condition $\eta\lamhat_i < \beta$ (so
$\tbeta_i = \beta - \eta\lamhat_i > 0$). Suppose that along the
eigendirection $\uhat_i$ of $\Sighat$, the normalized \DeltaMomentum
operates in the underdamped regime, i.e.\ $\Delta_i^{\mathrm{Delta}} < 0$.
Then, \emph{regardless of EMA's regime},
\begin{equation}
\rho_i^{\mathrm{Delta}} \;=\; \sqrt{\beta - \eta\lamhat_i} \;<\; \sqrt{\beta} \;\le\; \rho_i^{\mathrm{EMA}}.
\label{eq:strengthened-rate}
\end{equation}
The reduction is anisotropic, since directions with larger directional density
$\lamhat_i$ enjoy a strictly larger reduction.
\end{theorem}

\begin{proof}
Under stability, $\det A_i^{\mathrm{Delta}} = \tbeta_i > 0$ and
$\det A_i^{\mathrm{EMA}} = \beta > 0$ (\Cref{thm:detred}). Since
$\Delta_i^{\mathrm{Delta}} < 0$, the eigenvalues of $A_i^{\mathrm{Delta}}$
form a complex-conjugate pair with squared modulus $\det A_i^{\mathrm{Delta}}$,
hence $\rho_i^{\mathrm{Delta}} = \sqrt{\tbeta_i} = \sqrt{\beta - \eta\lamhat_i}$.
By Lemma~\ref{lem:det-bound} applied to $A_i^{\mathrm{EMA}}$,
$\rho_i^{\mathrm{EMA}} \ge \sqrt{\det A_i^{\mathrm{EMA}}} = \sqrt{\beta}$,
with equality only if EMA is also underdamped or critically damped along
$\uhat_i$. Since $\eta\lamhat_i > 0$,
$\sqrt{\beta - \eta\lamhat_i} < \sqrt{\beta} \le \rho_i^{\mathrm{EMA}}$.
\end{proof}

\begin{remark}[Explicit underdamped condition]\label{rem:underdamped-typicality}
The discriminant of $A_i^{\mathrm{Delta}}$ is
$\Delta_i^{\mathrm{Delta}} = (1+\tilde{\beta}_i-\alpha\eta\kappa_i)^2 - 4\tilde{\beta}_i$,
so the underdamped condition $\Delta_i^{\mathrm{Delta}}<0$ is equivalent to
\begin{equation}
\alpha\eta\kappa_i \;\in\; \Bigl( (1-\sqrt{\tilde{\beta}_i})^2,\; (1+\sqrt{\tilde{\beta}_i})^2 \Bigr).
\label{eq:underdamped-cond}
\end{equation}
The window has width $4\sqrt{\tilde{\beta}_i}$, approaching $4$ as $\tilde{\beta}_i\to 1$, so the band is wide whenever $\tilde{\beta}_i$ is not far from $1$. The same calculation with $\tilde{\beta}_i\!\to\!\beta$ yields the EMA condition $\alpha(1-\beta)\kappa_i\in((1-\sqrt{\beta})^2,(1+\sqrt{\beta})^2)$, and for matched $(\alpha,\beta)$ the two regimes share the same directions to leading order in $\eta\lamhat_i$, so \Cref{thm:strengthened-rate} applies on exactly the directions where EMA is also underdamped, which is the comparison of interest. For AK-AdamW the relevant $\alpha$ in \eqref{eq:underdamped-cond} is the per-direction \emph{effective} step size $\alpha/(\sqrt{v_i}+\varepsilon)$.
\end{remark}

\textbf{Regime of validity.}
Theorem~\ref{thm:strengthened-rate} restricts the comparison to the regime where
\DeltaMomentum is underdamped along $\hat{u}_i$, and EMA may be in any regime. Two
remarks. \emph{(i) Practical relevance.} High-density directions
($\hat{\lambda}_i$ near 1) typically operate in this regime under reasonable
hyperparameters, and these directions dominate early- and mid-training
loss-curve dynamics, making Theorem~\ref{thm:strengthened-rate} the operative regime in
practice. \emph{(ii) Honest scope.} When \DeltaMomentum is overdamped along
$\hat{u}_i$, the per-direction comparison depends jointly on trace and
determinant, and a determinant decrease at fixed trace can in principle
increase $\rho$.

\section{Proofs for Non-stationary Tracking Dynamics}
\label{app:tracking-proofs}
\label{app:tracking}

\subsection{Direction-selective Tracking under Distribution Shift}
\label{app:direction-selective-tracking}

\begin{proof}
\textbf{EMA.} $\E[M_t \mid M_{t-1}] = \beta M_{t-1} + (1-\beta)\Gbar
= \Gbar + \beta(M_{t-1} - \Gbar) = \Mstar + \beta(M_{t-1} - \Mstar)$.
Subtracting $\Mstar$ and projecting onto $u_i$ yields~\eqref{eq:tracking-ema}.

\textbf{\DeltaMomentum.} From~\eqref{eq:dm-factored},
$\E[M_t \mid M_{t-1}] = M_{t-1}(\beta I - \eta\Sigma_k) + \eta\Gbar_k$.
Using the fixed-point identity
$\eta\Gbar_k = \Mstar[(1-\beta)I + \eta\Sigma_k] = \Mstar - \Mstar(\beta I - \eta\Sigma_k)$:
\begin{align*}
\E[M_t - \Mstar \mid M_{t-1}]
&= M_{t-1}(\beta I - \eta\Sigma_k) + \eta\Gbar_k - \Mstar \\
&= M_{t-1}(\beta I - \eta\Sigma_k) - \Mstar(\beta I - \eta\Sigma_k) \\
&= (M_{t-1} - \Mstar)(\beta I - \eta\Sigma_k).
\end{align*}
Projecting onto $u_i$ and using $\Sigma_k u_i = \lambda_i u_i$ yields
$\E[T_{t,i}] = \tbeta_i T_{t-1,i}$.
\end{proof}

\subsection{Uniform Tracking-rate Bounds (Normalized Variant)}

For completeness we record the direction-uniform tracking bounds for the normalized variant of \DeltaMomentum. They follow directly from \Cref{thm:tracking} and Proposition~\ref{prop:bounded-sigma-hat}.

\begin{corollary}[Uniform tracking-rate bounds, normalized variant]
\label{cor:uniform-tracking}
For the normalized variant, $\hat{\lambda}_i \in [0,1]$ implies
\begin{equation}
\beta - \eta \;\le\; \tilde{\beta}_i \;\le\; \beta \quad \text{for all } i,
\label{eq:uniform-tracking}
\end{equation}
so 
the fastest-tracking direction contracts at rate $\beta - \eta$, and the slowest is bounded above by EMA's uniform rate $\beta$
(a width- and layer-independent guarantee). 
Moreover, the directional-density-weighted mean contraction rate is
\begin{equation}
\sum_i \hat{\lambda}_i \tilde{\beta}_i \;=\; \beta - \eta\,\|\hat{\Sigma}\|_F^2,
\label{eq:weighted-mean}
\end{equation}
which by Cauchy--Schwarz satisfies $\sum_i \hat{\lambda}_i \tilde{\beta}_i \le \beta - \eta/n$, with equality iff $\hat{\Sigma}$ is isotropic. Anisotropy of $\hat{\Sigma}$ strictly improves the density-weighted mean tracking rate beyond the isotropic baseline $\beta - \eta/n$.
\end{corollary}

\label{app:uniform-tracking-rate}
\begin{proof}
The bracket~\eqref{eq:uniform-tracking} is direct from
$0 \le \lamhat_i \le 1$. For the density-weighted mean,
$\sum_i \lamhat_i\tbeta_i = \beta\sum_i \lamhat_i - \eta\sum_i \lamhat_i^2
= \beta - \eta\,\|\Sighat\|_F^2$, using $\tr(\Sighat) = 1$. The
Cauchy--Schwarz inequality
$1 = (\sum_i \lamhat_i)^2 \le n\sum_i \lamhat_i^2$ gives
$\|\Sighat\|_F^2 \ge 1/n$, with equality iff all $\lamhat_i$ are equal.
\end{proof}

\section{Scope of the Theoretical Assumptions}
\label{app:scope}

The results of \Cref{sec:methodology} carry different amounts of assumption. This appendix states which one describes a training run.

\Cref{lem:transport} needs only the key second moment and is an exact property of the deployed update, so it carries to any architecture, loss, and data. \Cref{thm:tracking} lets both $\Ghat$ and $\Sighat$ move, imposes no regime restriction, and gives strictly faster expected tracking-error contraction along every direction with positive input density. The runs of \Cref{sec:experiments} live there, so \Cref{thm:tracking} and not \Cref{thm:fixedpoint} is what describes them.

\Cref{thm:fixedpoint} is the only result that needs quasi-staticity, and it needs half of it. The assumption covers two quantities, $\Ghat$ and $\Sighat$, and $\Sighat$ alone sets the per-direction decay rates. Drift in $\Ghat$ only makes the buffer lag a moving target, which is the tracking error \Cref{thm:tracking} already bounds. The real condition is therefore that $\Sighat$ moves slowly compared with the buffer memory window $\tau_i$ of \eqref{eq:horizon}, which is at most $100$ steps at $\beta = 0.99$. The condition is local in time. It asks that $\Sighat$ hold still over one window and reapplies as the window slides, so \Cref{thm:fixedpoint} describes the current phase of training rather than the whole run. This is the same locality that K-FAC and Shampoo rely on when they reuse an input-side factor for tens to hundreds of steps. It holds most comfortably after warmup and is weakest early in training and under abrupt input shifts, which are the cases \Cref{thm:tracking} covers.

The linear-regression analysis of \Cref{sec:linear-regression-dynamics} is idealized by design. It exhibits a clean setting in which a strict per-direction improvement can be proved exactly rather than accounting for transformer training, and we know of no comparable result separating momentum schemes on non-convex losses. What carries to a transformer is the mechanism and not the closed form, and \Cref{lem:loss-free-det-main} and \Cref{prop:nc1-main} are the two statements that carry it. Appendix~\ref{app:beyond-linreg} works out what they give beyond the squared loss, and \Cref{sec:diag} measures on the real 24-layer run the quantities that \Cref{lem:transport} and \Cref{thm:tracking} govern.

\section{Guarantees Beyond Standard Linear Regression}
\label{app:beyond-linreg}

\Cref{sec:linear-regression-dynamics} proves a strict per-direction improvement on standard linear regression. This appendix shows that the guarantee survives in two settings that the main text does not cover, high-dimensional regression with $p > n$ and canonical generalized linear models including logistic regression. Appendix~\ref{app:nonconvex} then treats the non-convex losses deep networks produce, which are reached through the tracking result of \Cref{sec:non-stationary-tracking-dynamics} rather than through the determinant. Throughout we work in the deployed normalized-key update based on \Cref{rem:normalized-key}.

\textbf{One ratio organizes all three.}
Along direction $u_i$ the delta rule forgets at $(1-\beta) + \eta\lamhat_i$ per step against EMA's $(1-\beta)$, a ratio of
\begin{equation}
1 + r_i, \qquad r_i := \frac{\eta\lamhat_i}{1-\beta},
\label{eq:ratio}
\end{equation}
and every gain below comes from that one factor. It shortens the memory window, it speeds relaxation toward a moving target, it shrinks the steady lag bound, and it scales direction $i$ of the implicit preconditioner $(\mu I + \Sighat)^{-1}$ by $1/(1+r_i)$, which lowers the effective condition number on the data span.

\subsection{A Loss-Free Determinant Lemma}
\label{app:lemma-a}

The per-direction iteration of \eqref{eq:iter-matrices} was derived for a squared loss, but its determinant never uses the loss. Making that explicit is what lets the same comparison run in every setting below.

\begin{lemma}[Loss-free determinant, restating \Cref{lem:loss-free-det-main}]
\label{lem:loss-free-det}
Along $u_i$, one step maps the pair (parameter error, buffer component) through
\begin{equation}
A = \begin{pmatrix} 1 - \alpha c & -\alpha \tbeta_i \\ c & \tbeta_i \end{pmatrix},
\label{eq:loss-free-A}
\end{equation}
where $\alpha$ is the step size and $c$ is the per-direction coupling from the parameter error into the buffer, equal to $\eta$ times the local curvature along $u_i$. This is \eqref{eq:iter-matrices} with $c$ left arbitrary, and $c = \eta\kappa_i$ recovers it. Then
\begin{equation}
\det A = \tbeta_i \quad \text{for \DeltaMomentum}, \qquad \det A = \beta \quad \text{for EMA momentum},
\end{equation}
for any $c$ and any $\alpha$. The determinant depends only on the memory statistic, never on the loss that produced $c$. Consequently, by \Cref{lem:det-bound},
\begin{equation}
\min_\alpha \rho(A) = \sqrt{\det A},
\end{equation}
so the best achievable per-direction rate falls from $\sqrt{\beta}$ to $\sqrt{\tbeta_i}$ on every direction with $c > 0$, whatever loss produced it.
\end{lemma}

\begin{proof}
Direct computation gives $\det A = \tbeta_i(1-\alpha c) + \alpha\tbeta_i c = \tbeta_i$, and the same computation with $\tbeta_i$ replaced by $\beta$ gives $\det A = \beta$. Neither retains $c$. \Cref{lem:det-bound} gives $\rho(A) \ge \sqrt{\det A}$ whenever $\det A > 0$, with equality exactly when the discriminant is non-positive, and \Cref{rem:underdamped-typicality} shows the discriminant is non-positive on a band of $\alpha c$ of width $4\sqrt{\tbeta_i}$, which is non-empty, so the bound is attained at the tuned step size.
\end{proof}

We state the comparison as a bound on the best achievable rate rather than as a claim at fixed $\alpha$. At a step size held fixed across the two methods, a smaller determinant does not by itself force a smaller spectral radius, since an overdamped iteration can move the other way, which is the same scope \Cref{thm:strengthened-rate} states.

\subsection{High-dimensional Linear Regression}
\label{app:high-dim}

In the regime $p > n$, where $p$ is the input dimension and $n$ the number of samples, the guarantee survives and sharpens.

Off the key span the rule is inert. If $\xhat^\top u = 0$ for every observed key, then $\Sighat u = 0$ and $\Ghat u = \E[\delta(\xhat^\top u)] = 0$, so both the correction term and the gradient term vanish along $u$. Neither the EMA buffer nor the \DeltaMomentum buffer puts mass on a direction the data has not queried, and the difference between them is confined to the queried subspace.

The fixed point stays well posed. \Cref{cor:fixedpoint-normalized} gives $\Mstar = \Ghat(\mu I + \Sighat)^{-1}$, which is defined for rank-deficient $\Sighat$, exactly the regime in which the unregularized whitening $\Ghat\Sighat^{-1}$ is undefined. The rule therefore regularizes on its own, with the ridge $\mu = (1-\beta)/\eta$ supplied by the decay rather than chosen separately.

The gain grows with spread. Over the queried directions the improvement is the ratio $(\mu + \lamhat_{\max})/(\mu + \lamhat_{\min})$, so it is largest when the queried spectrum is spread out, which is the typical spiked scenario in high dimension. A ceiling is available for the normalized variant alone, since $\lamhat_i \in [0,1]$ gives
\begin{equation}
\mathrm{cond}(\mu I + \Sighat) \;\le\; \frac{\mu + 1}{\mu} \;=\; 1 + \frac{\eta}{1-\beta},
\label{eq:cond-ceiling}
\end{equation}
a bound that is independent of width, of layer, and of the data.

\subsection{Canonical GLMs}
\label{app:glm}

For a canonical generalized linear model with inverse link $f$, the per-sample gradient is $(f(w^\top x) - y)x^\top$, which is the same key-value structure the method is built on. \Cref{thm:fixedpoint}, \Cref{lem:transport}, and \Cref{thm:tracking} therefore hold unchanged, since their proofs use only $\Ghat$ and $\Sighat$ and never a squared loss.

For the rate, linearize the joint dynamics of the parameter error and the buffer at the optimum $w_\star$. The scalar curvature $c$ of \Cref{lem:loss-free-det} becomes the matrix
\begin{equation}
K = \E\!\left[f'(w_\star^\top x)\, x \xhat^\top\right] = \E\!\left[f'(w_\star^\top x)\, \|x\|\, \xhat\xhat^\top\right],
\end{equation}
which is symmetric and positive semidefinite, using $x\xhat^\top = \|x\|\,\xhat\xhat^\top$ and $f' > 0$ for canonical links. If $K$ and $\Sighat$ share eigenvectors they diagonalize together, so the system splits into one block per direction, each a copy of \eqref{eq:loss-free-A} with the matching eigenvalue of $K$ in place of $c/\eta$, and \Cref{lem:loss-free-det} applies block by block. The fixed-point and tracking statements need no such condition.

\subsection{Nonlinear Deep Neural Network Training}
\label{app:nonconvex}

For non-convex deep losses, tracking is the right tool, because \Cref{thm:tracking} constrains the recursion that produces the buffer and places no condition on the loss surface, so convexity never enters. We strengthen it with two results, stated against each buffer's own target so that scale differences between the two targets do not confound the comparison. We claim no global rate.

\begin{proposition}[Steady lag, restating \Cref{prop:nc1-main}]
\label{prop:nc1}
Suppose the target of the buffer drifts by at most $\Delta_i$ per step along $u_i$. Then the expected tracking error satisfies
\begin{equation}
\limsup_{t\to\infty} \|\E[T_{t,i}]\| \;\le\; \frac{\tbeta_i \Delta_i}{1 - \tbeta_i}
\qquad\text{against EMA's}\qquad
\frac{\beta\Delta_i}{1-\beta}.
\end{equation}
Since $\tbeta_i < \beta$ on every direction with positive input density and $t \mapsto t/(1-t)$ is increasing on $[0,1)$, the bound is strictly smaller there, and the relaxation-time ratio is $1 + r_i$.
\end{proposition}

\begin{proof}
By \Cref{thm:tracking}, one step contracts the tracking error by $\tbeta_i$, and the drift adds at most $\Delta_i$ before the contraction, so $\|\E[T_{t,i}]\| \le \tbeta_i(\|\E[T_{t-1,i}]\| + \Delta_i)$. The fixed point of this recursion is $\tbeta_i\Delta_i/(1-\tbeta_i)$ and the recursion is a contraction, so the limit superior is bounded by it. The EMA statement is the same argument with $\tbeta_i$ replaced by $\beta$, which is \eqref{eq:tracking-ema}. Monotonicity of $t/(1-t)$ and $\tbeta_i < \beta$ give the comparison. The relaxation time is $1/(1-\tbeta_i)$ against $1/(1-\beta)$, a ratio of $((1-\beta)+\eta\lamhat_i)/(1-\beta) = 1 + r_i$.
\end{proof}

\begin{corollary}[Uniform lag]
\label{cor:nc2}
If the drift follows directional density, $\Delta_i = c_0 \lamhat_i$, then the \DeltaMomentum lag is bounded by $\beta c_0/\eta$ uniformly over directions, while EMA's grows linearly in the density as $\beta c_0 \lamhat_i/(1-\beta)$.
\end{corollary}

\begin{proof}
Substituting into \Cref{prop:nc1}, the lag is $\tbeta_i c_0\lamhat_i/((1-\beta) + \eta\lamhat_i)$. Bounding the numerator by $\beta c_0 \lamhat_i$ and the denominator below by $\eta\lamhat_i$ gives $\beta c_0/\eta$, which does not involve $i$. The EMA bound is $\beta c_0\lamhat_i/(1-\beta)$, which is linear in $\lamhat_i$.
\end{proof}

\textbf{The bridge back to the loss.}
Split the estimation error into two parts, $M - \Ghat = (M - \Mstar) + (\Mstar - \Ghat)$. The first term is the lag, which is how far the buffer trails its own moving target. \Cref{prop:nc1} lowers its steady bound on every direction and \Cref{cor:nc2} caps it where EMA's grows with the density. The second term is the designed preconditioning bias of \Cref{thm:fixedpoint} rather than an error. So \Cref{prop:first-order-optimal} turns reduced lag on the high-density directions, which \Cref{fig:top1} shows dominate the input variance, into reduced first-order shortfall in the realized loss decrease. \Cref{sec:diag} is the empirical form of that bridge, since it measures the gradient-estimator quality, the function-space prediction error, and the input-feature conditioning on the real 24-layer run.

\section{Mixed-Modality Inputs}
\label{app:multimodal}

A natural question is what happens when the input features are of mixed types, such as tabular data, text, and images together, which carry different natural metrics. The short answer is that nothing in \Cref{sec:methodology} changes, because \DeltaMomentum acts inside a layer on post-embedding activations and never sees raw input types.

The theory is untouched. Every result in \Cref{sec:methodology} uses the input distribution only through $\Sighat$ and $\Ghat$. The key second moment of a mixture is the density-weighted combination of the second moments of its components, so heterogeneous input produces a valid $\Sighat \succeq 0$ and every statement holds for any such $\Sighat$. Mixing modalities changes the key spectrum, which is the very thing the mechanism adapts to.

There is one metric, not several. On the input side, after the embedding every modality lives in the same activation space under the same inner product. Per-sample $\ell_2$ normalization also removes residual scale and unit differences before they reach the buffer while preserving directional density, which is scale free by \Cref{prop:bounded-sigma-hat}. On the output side the loss enters only through $\delta$, so a different loss or metric per modality changes what is stored, not how it is stored, and no theorem in \Cref{sec:methodology} uses a particular loss. An embedding is needed only in the sense that any network already has one, so no extra machinery and no new theory are required. Our experiments span two modalities trained separately, and an interleaved study is noted as future work in \Cref{sec:conclusion}.

\section{Interaction with the Second Moment of AdamW}
\label{app:second-moment}

The delta correction and the second moment of AdamW act on different structure, so they compose rather than duplicate. This appendix gives the argument in full. We present it as a discussion rather than as a theoretical claim, since the central step is a heuristic factorization that we do not prove and that is exact only under an independence assumption a trained network does not satisfy.

\textbf{The two act in different bases.}
The second moment $v$ is diagonal in the coordinate basis and rescales entries of the update. The delta correction is diagonal in the eigenbasis of $\Sighat$ and rescales input directions. Key normalization keeps the roles apart, since $\Sighat$ has unit trace and therefore carries direction only, leaving overall scale to Adam.

\textbf{Where the input side enters each.}
Heuristically, under $g = \delta x^\top$ with $\delta$ and $x$ independent, the second moment factorizes as $v_{ij} \approx \E[\delta_i^2]\,\E[x_j^2]$, so the input-side factor that Adam sees is the diagonal of the input second moment, while the delta rule uses all of $\Sighat$. The two therefore coincide only when $\Sighat$ is near isotropic or coordinate aligned. \Cref{fig:cond} rules that out in our models, where the per-layer input-feature condition numbers sit between $10^4$ and $10^{10}$ throughout training. The independence assumption behind the factorization does not hold in a trained network, which is why we do not state this as a claim, but the conclusion it points to does not depend on the assumption being exact, only on $\Sighat$ being far from isotropic.

\textbf{What the delta rule adds on top of the second moment.}
Both arms of every comparison in \Cref{sec:experiments} use the unchanged AdamW second moment, so whatever the delta rule contributes is measured on top of the scaling of Adam. The separations reported in \Cref{sec:diag-feat} are measured on the input-feature covariance, which $v$ registers only through its diagonal, so they are outside what $v$ can produce on its own.

\textbf{The one place they overlap.}
The two do overlap in the overall size of $M$ relative to $\sqrt{v}$. Learning-rate tuning absorbs that overlap, which is the reason the tuned peak rate for AK-AdamW sits more than ten times below the AdamW value rather than near it (Appendix~\ref{app:hp-sweep}). We read the offset as a scale correction and not as evidence that the two preconditioners are competing for the same structure.

\section{Deployed \DeltaMomentum Algorithm}
\label{app:deployed}

This appendix states the algorithm that produces every number reported in \Cref{sec:experiments}, so that the cost analysis of Appendix~\ref{app:flops-memory} and the empirical claims refer to the same object.

Per linear layer, with a batch of $N$ tokens giving inputs $X \in \R^{N\times n}$, their normalized keys $\Xhat$, and output-side errors $\Delta \in \R^{N \times m}$, write
\begin{equation}
G = \frac{\Delta^\top X}{N} \in \R^{m\times n},
\qquad
\Sighat = \frac{\Xhat^\top \Xhat}{N} \in \R^{n \times n},
\end{equation}
and update the buffer by
\begin{equation}
\boxed{\;\; M \;\leftarrow\; \beta M + \eta\bigl(G - M\Sighat\bigr). \;\;}
\label{eq:deployed}
\end{equation}
This is the batch form of \eqref{eq:dm-update} with the normalized key in the correction term. AK-AdamW then applies the unchanged AdamW machinery to $M$, that is the second moment $v_t = \beta_2 v_{t-1} + (1-\beta_2)g_t^{\odot 2}$ taken on the ordinary weight gradient, bias correction, decoupled weight decay, and the preconditioned update $\Mhat_t \oslash (\sqrt{\hat v_t} + \varepsilon)$. AK-SGD applies $\theta \leftarrow \theta - \alpha M$ directly. The CIFAR-10 MLP run of Section~\ref{sec:main-deltaSGD} instead forms the gradient term against the normalized keys, $G = \Delta^\top \Xhat/N$, which makes \eqref{eq:deployed} the batch form \eqref{eq:nk-delta} of the normalized-key rule \eqref{eq:dm-norm-update}.

Two implementation points matter for the cost.

\emph{The gradient term is free.} $G$ is the weight gradient that the backward pass already computes, and the AdamW second moment is taken on the same $G$, so the only optimizer-specific work is the correction $M\Sighat$. Appendix~\ref{app:flops-memory} counts exactly that.

\emph{The contraction order is chosen per layer.} The product $M\Sighat$ can be formed either as $M(\Xhat^\top\Xhat)$, which builds the Gram matrix once and reuses it across every layer reading the same input, or as $(M\Xhat^\top)\Xhat$, which never materializes an $n \times n$ matrix. The first is cheaper for square and tall layers and the second for wide ones, and the dispatch rule is given in Appendix~\ref{app:flops-memory}. The second moment $\Sighat$ is taken on a tensor the step already forms and adds no GEMM of its own.

\section{FLOPs Count and Measurement Setup}
\label{app:flops-memory}

\subsection{Notation and FLOP Convention}

For a single linear layer $y = Wx$ with weight $W \in \mathbb{R}^{m \times n}$
where $m = \dout$ and $n = \dinp$, batched over $N = B \cdot L_{\mathrm{seq}}$ tokens, with $B$ the batch size in sequences and $L_{\mathrm{seq}}$ the sequence length:
\begin{center}
\begin{tabular}{cll}
\toprule
Symbol & Meaning & Shape \\
\midrule
$X$ & Input activations & $N \times n$ \\
$\nablaY$ & Output-side gradient & $N \times m$ \\
$W$ & Weight matrix & $m \times n$ \\
$M$ & Momentum buffer & $m \times n$ \\
$V$ & Second-moment buffer (AdamW) & $m \times n$ \\
$\Xhat$ & Row-normalized inputs, $\xhat_i = x_i / \|x_i\|_2$ & $N \times n$ \\
$d$ & Transformer hidden dimension & --- \\
$\gamma$ & Gated-MLP expansion ratio, $d_{\mathrm{ffn}} = \gamma\, d$ & --- \\
$L_{\mathrm{seq}}$ & Sequence length, distinct from the layer count $L$ & --- \\
\bottomrule
\end{tabular}
\end{center}

We adopt the standard convention that one fused multiply-add (FMA) counts
as $2$ FLOPs. Hence a matrix multiply of an $a\times b$ and $b\times c$
matrix is counted as $2abc$ FLOPs.

\subsection{Standard Linear-Layer FLOPs}

The baseline cost of a linear layer (forward + backward, excluding the
optimizer step since it is comparatively negligible) is
\begin{equation}\label{eq:baseline-layer}
\FLOP_{\text{baseline}}^{\text{layer}}
= \underbrace{2mnN}_{\text{forward } Y = X W^\top}
+ \underbrace{2mnN}_{\nabla X = \nablaY\, W}
+ \underbrace{2mnN}_{\nabla W = \nablaY^\top X}
= 6mnN.
\end{equation}

\subsection{The Normalized-Key \DeltaMomentum Update}

\begin{definition}[Normalized-Key \DeltaMomentum]
The normalized-key delta-rule update is
\begin{equation}\label{eq:nk-delta}
\boxed{\;\;
M_{t+1} = \beta\, M_t + \frac{\eta}{N}\Bigl(\nablaY^\top \Xhat - M_t\, \Xhat^\top \Xhat\Bigr),
\;\;}
\end{equation}
or per-sample,
\begin{equation}
M_{t+1} = \beta M_t + \eta \bigl(v_t - M_t \xhat_t\bigr) \xhat_t^\top,
\qquad
v_t \coloneqq \delta_t,
\end{equation}
where $G_t = \delta_t x_t^\top$ is the per-sample gradient and $\delta_t = \partial \mathcal{L} / \partial y_t$.
\end{definition}

The factored implementation evaluates~\eqref{eq:nk-delta} as
\begin{equation}\label{eq:factored}
\Delta M = \tfrac{1}{N}\bigl(\nablaY^\top \Xhat - M\Xhat^\top \Xhat\bigr)
\equiv \tfrac{1}{N}\bigl(\nablaY^\top - M\Xhat^\top\bigr)\Xhat .
\end{equation}

\subsection{Per-Layer Cost of the Correction}

The deployed update is \eqref{eq:deployed}, and Appendix~\ref{app:deployed} explains why only the correction $M\Sighat$ is optimizer-specific. This subsection counts that correction per layer, then Appendix~\ref{app:block-flops} aggregates it over a block.

\subsubsection{Cost of Key Normalization}

\begin{lemma}[Normalization FLOPs]\label{lem:norm-cost}
Computing $\Xhat$ from $X$ costs
\begin{equation}
\FLOP_{\mathrm{norm}}(n, N) = 3nN + 2N \approx 3nN .
\end{equation}
\end{lemma}

\begin{proof}
Per row $i$, computing $\|x_i\|_2^2 = \sum_j x_{ij}^2$ takes $n$ multiplies plus $n-1$ adds, which is $2n - 1 \approx 2n$ FLOPs, the square root takes $1$, the reciprocal takes $1$, and scaling all $n$ entries takes $n$ multiplies. Aggregating over $N$ rows gives $N(3n+2) = 3nN + 2N$.
\end{proof}

This is sub-leading against a single GEMM at $2mnN$ by a factor of $3/(2m)$, which is about $0.15\%$ at $m = 1024$. The normalization is computed once per distinct input and shared by every layer reading it.

\subsubsection{Two Contraction Orders and the Dispatch Rule}
\label{sec:dispatch}

The correction $M\Sighat$ admits two evaluation orders with different costs.

\emph{Gram order.} Form $\Ghat_{\mathrm{gram}} = \Xhat^\top\Xhat$ at $2n^2N + n^2$, then $M\,\Ghat_{\mathrm{gram}}$ at $2mn^2$, for a total of $2n^2N + 2mn^2$ up to lower-order terms. The Gram matrix is formed once per distinct input and reused by every layer that reads that input, so the $2n^2N$ term is shared while the $2mn^2$ term is per layer. Its spectrum lies in $[0, N]$ before the $1/N$ scaling, so the product is numerically stable.

\emph{Through-$N$ order.} Form $M\Xhat^\top$ at $2mnN$, then multiply by $\Xhat$ at $2mnN$, for a total of $4mnN$, never materializing an $n \times n$ matrix.

Comparing the two leading terms, $2n^2N$ against $4mnN$, the Gram order is cheaper exactly when $n < 2m$. This gives the dispatch rule used throughout.
\begin{equation}
M\Sighat \;=\;
\begin{cases}
M\,(\Xhat^\top \Xhat) & \text{if } n < 2m \quad \text{(Gram order)},\\[2pt]
(M \Xhat^\top)\,\Xhat & \text{otherwise} \quad \text{(through-}N\text{ order)}.
\end{cases}
\label{eq:dispatch}
\end{equation}

\subsection{Per-Block Count and the Overhead Band}
\label{app:block-flops}

\subsubsection{Baseline Per-Block FLOPs}

A standard transformer block with a gated MLP of expansion ratio $\gamma$ holds four attention projections at $d\times d$, two gated-MLP matrices at $\gamma d \times d$, and one down projection at $d \times \gamma d$. Using \eqref{eq:baseline-layer} for each,
\begin{equation}
\FLOP^{\text{baseline}}_{\text{block}}(\gamma)
\;=\; \underbrace{(24 + 18\gamma)\, d^2 N}_{\text{linear (QKVO + gated MLP)}}
\;+\; \underbrace{12\, d L_{\mathrm{seq}} N}_{\text{attention scores/values}} .
\label{eq:baseline-block}
\end{equation}
For the Llama-2 ratio $\gamma = 8/3$, the linear term is $72\,d^2N$.

\subsubsection{Aggregating the Correction}

Applying \eqref{eq:dispatch} layer by layer, and sharing each Gram matrix across the layers that read the same input, gives the following.

\begin{center}
\begin{tabular}{lllcc}
\toprule
Group & Shape & Order & $d^2N$ coeff. & $d^3$ coeff. \\
\midrule
$W_Q, W_K, W_V$ & $d \times d$ & Gram, shared over one input & $2$ & $6$ \\
$W_O$ & $d \times d$ & Gram & $2$ & $2$ \\
$W_{\mathrm{up}}, W_{\mathrm{gate}}$ & $\gamma d \times d$ & Gram, shared over one input & $2$ & $4\gamma$ \\
$W_{\mathrm{down}}$ & $d \times \gamma d$ & through-$N$ & $4\gamma$ & $0$ \\
\midrule
\textbf{Total} & & & $\boldsymbol{6 + 4\gamma}$ & $\boldsymbol{8 + 4\gamma}$ \\
\bottomrule
\end{tabular}
\end{center}

The first three rows each contribute one shared $2d^2N$ term rather than one per layer, because $W_Q, W_K, W_V$ read one and the same input and $W_{\mathrm{up}}, W_{\mathrm{gate}}$ read another. Sharing applies to the Gram matrix only. The multiplication by $M$ that follows it is per layer at $2mn^2$, which carries no factor of $N$ and is what the $d^3$ column collects. The down projection maps the expanded hidden state back to the model width, so $m = d$ and $n = \gamma d$ there. Every gated MLP in use has $\gamma \ge 2$, which makes $n \ge 2m$, so this layer takes the through-$N$ order at $4mnN = 4\gamma d^2 N$. Key normalization costs $3nN$ per distinct input by \Cref{lem:norm-cost}, which over the four inputs above, of dimensions $d$, $d$, $d$, and $\gamma d$, gives $(9 + 3\gamma)dN$. The block count is therefore
\begin{equation}
\Delta\FLOP_{\text{block}}(\gamma)
= (6 + 4\gamma)\, d^2 N \;+\; (8 + 4\gamma)\, d^3 \;+\; (9 + 3\gamma)\, dN .
\label{eq:block-overhead-deployed}
\end{equation}
The two lower-order terms enter the ratio as $(8+4\gamma)d/N$ and $(9+3\gamma)/d$, which at $N = 524{,}288$ tokens per step and our widths of 384, 1024, and 1792 move it by about a tenth of a percentage point, so the leading term is the block-level number.

\subsubsection{The Overhead Band}

\begin{theorem}[Per-block overhead of the deployed rule]
\label{thm:deployed-overhead}
Under the dispatch rule \eqref{eq:dispatch} with normalized keys and $\gamma \ge 2$, the per-block cost of the correction relative to the linear forward and backward of \eqref{eq:baseline-block} is
\begin{equation}
\boxed{\;\;
f(\gamma) \;=\; \frac{6 + 4\gamma}{24 + 18\gamma},
\;\;}
\end{equation}
which is strictly decreasing on $\gamma \ge 0$ and therefore satisfies
\begin{equation}
22.22\% \;=\; \tfrac{4}{18} \;<\; f(\gamma) \;\le\; \tfrac{6}{24} \;=\; 25.0\% ,
\end{equation}
with $f(8/3) = 50/216 = 23.148\%$ at the Llama-2 ratio used in all our models.
\end{theorem}

\begin{proof}
The ratio follows from \eqref{eq:block-overhead-deployed} and \eqref{eq:baseline-block} by dropping the two lower-order terms. The endpoints are $f(0) = 6/24 = 25.0\%$ and $f(\gamma) \to 4/18 = 22.22\%$ as $\gamma \to \infty$. Endpoints alone do not give a band, so we add that $f'(\gamma) = -12/(24+18\gamma)^2 < 0$ for every $\gamma \ge 0$, which makes $f$ strictly decreasing and puts every value between the two. The down projection is the only layer whose order depends on $\gamma$, so the count above holds for $\gamma \ge 2$ and the attainable values sit in the lower part of the range.
\end{proof}

The band is what makes the cost predictable across architectures. Any gated-MLP shape lands inside a range of under three percentage points, so the block-level number barely moves when the expansion ratio changes. Every gated MLP in use has $\gamma \ge 2$, so the attainable values sit in the narrower range from $22.22\%$ to $f(2) = 23.33\%$, and the $25.0\%$ endpoint is the limit at $\gamma = 0$ rather than a shape any such block takes.

\subsection{Realized Per-Step Overhead and Measurement Setup}
\label{app:realized-flops}

\Cref{thm:deployed-overhead} is a block-level statement about linear layers. 
Two effects at finite scale place the realized per-step overhead below it.

\begin{enumerate}
\item the attention $\bigO(d L_{\mathrm{seq}} N)$ term of \eqref{eq:baseline-block} inflates the baseline at $L_{\mathrm{seq}} = 2048$.
\item the output projection maps the width $d$ to the vocabulary size $n_{\mathrm{vocab}}$, so it adds $6\,d\,n_{\mathrm{vocab}}N$ to the baseline but only about $2d^2N$ to the correction, since its Gram matrix is $d \times d$. It therefore enters the ratio at about $d/(3\,n_{\mathrm{vocab}})$, which is $0.4\%$, $1.1\%$, and $1.9\%$ at our three widths and far below the block value. The input embedding is a lookup rather than a GEMM and enters neither side of the count.
\end{enumerate}
Both enlarge the denominator of the overhead ratio far more than the numerator, and both shrink relative to the leading $d^2N$ term as width grows, which is why the realized number rises with scale toward the band.

\textbf{Measurement setup.}
We do not estimate the realized overhead from the closed form. Our training suite counts it with PyTorch's FLOP counter over the forward and backward of a full optimizer step, which charges every matrix multiplication there, including those that form the correction, and charges attention densely, without halving it for the causal mask and including the recomputation of the attention scores in its backward pass, which gives $14\,dL_{\mathrm{seq}}N$ per layer against the $12\,dL_{\mathrm{seq}}N$ of \eqref{eq:baseline-block}. The counter does not charge elementwise work, so the key normalization adds time but no counted FLOPs, and for Gram-order layers the once-per-step product of $M$ with the Gram matrix runs in the optimizer step, outside the counted region. These are the two lower-order terms of the per-block count above, which move the ratio by about a tenth of a percentage point. The count does not depend on how a step is split into microbatches, and the cost axes count training FLOPs only, since evaluation adds the same amount to both arms. All three models use $L = 24$, $d_{\mathrm{head}} = 64$, $n_{\mathrm{vocab}} = 32{,}000$, untied input and output embeddings, and $\gamma \approx 2.67$, with $d_{\text{ffn}}  = \lfloor 2.67 d \rfloor$, $B = 256$ and $L_{\mathrm{seq}} = 2048$, hence $N = 524{,}288$ tokens per step. Step time is measured in the training loop on a single H200, with activation tracing for monitoring turned off and the same microbatch size in both arms. Each value comes from at least three repeated measurements of a few hundred steps past warmup, each summarized by the median of its fastest quarter of steps, leaving out steps next to evaluations and checkpoint writes. Cumulative wall clock in \Cref{fig:wallclock} is this per-step time times the step count, so it covers training steps only and leaves out startup, evaluations and checkpoint writes, which add less than $1\%$ to each run in the figure.

\textbf{Results.}
\begin{center}
\begin{tabular}{l c c c c}
\toprule
Scale & $d$ & $d_{\mathrm{ffn}}$ & Realized per-step FLOP overhead & Measured step time \\
\midrule
67M  & 384  & 1025 & $\mathbf{10.0\%}$ & $1.1526\times$ AdamW \\
370M & 1024 & 2734 & $\mathbf{15.5\%}$ & $1.1775\times$ AdamW \\
1B   & 1792 & 4784 & $\mathbf{18.1\%}$ & $1.1844\times$ AdamW \\
\bottomrule
\end{tabular}
\end{center}

The counted overhead rises with width across all three scales, from $10.0\%$ to $18.1\%$, approaching the band of \Cref{thm:deployed-overhead} from below as the two finite-scale effects above shrink relative to the leading $d^2N$ term. Measured step time rises with it, from $1.1526\times$ to $1.1844\times$, and sits above the count by $5.3$, $2.2$, and $0.3$ percentage points at 67M, 370M, and 1B. The gap comes from the rate at which the correction runs rather than from its amount of work. Every GEMM of the correction runs on tensor cores, but at small width the through-$N$ product of the down projection, whose output has only $d$ rows, is too thin to fill the GPU, and the key normalization reads and writes the activations while adding almost no FLOPs. The AK-AdamW step runs at $95.4\%$ of the FLOP rate of the AdamW step at 67M. Both effects shrink as width grows, and this rate reaches $98.1\%$ at 370M and $99.7\%$ at 1B, where step time tracks the count. Reported step times use our PyTorch implementation, which compiles the key normalization into one fused kernel and forms each correction product with a single GEMM call. A kernel that normalizes keys as it loads them, tiled for the thin products at small width, would narrow the remaining gap.

\textbf{What this costs at matched loss.}
\Cref{tab:cost} reports the end-to-end consequence, measured over the full horizon with identical data order and initialization in both arms. Averaged over the matched validation-loss levels past the crossing point on each axis, AK-AdamW reaches the same loss with 35.1\% fewer FLOPs and 32.0\% less wall clock at 67M, and 5.6\% fewer FLOPs and 4.0\% less wall clock at 370M. \Cref{fig:wallclock} shows the wall-clock curves. Because the correction is paid from the first step, the two arms cross once on each cost axis and AK-AdamW is ahead of AdamW everywhere after that point. The cost curves are the seed means of \Cref{fig:main-loss}, since every seed has the same cost per step.

\begin{figure}[h]
  \centering
  \begin{subfigure}[t]{0.48\linewidth}
    \includegraphics[width=\linewidth]{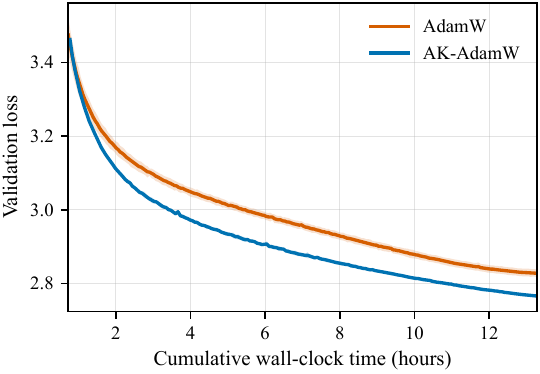}
    \caption{67M.}
  \end{subfigure}\hfill
  \begin{subfigure}[t]{0.48\linewidth}
    \includegraphics[width=\linewidth]{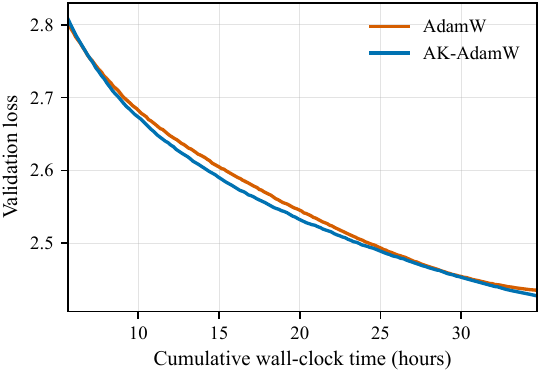}
    \caption{370M.}
  \end{subfigure}
  \caption{Validation loss against cumulative wall clock on a single H200, from step $2{,}000$ to the end of the AdamW run. Curves are the seed means of \Cref{fig:main-loss} with per-seed bands, and cumulative wall clock is the per-step time of Appendix~\ref{app:realized-flops} times the step count, with the same microbatch size in both arms. At 67M AK-AdamW is ahead throughout. At 370M it starts behind, because the correction is paid from the first step, crosses near $7.0$ hours, and stays ahead for the rest of the run. When AdamW finishes, AK-AdamW is $0.062$ nats lower at 67M and $0.008$ nats lower at 370M.}
  \label{fig:wallclock}
\end{figure}

\subsection{Peak-Memory Derivation}

\subsubsection{Persistent State}

\begin{center}
\begin{tabular}{l c c}
\toprule
Optimizer & Buffers & Bytes/param (fp32) \\
\midrule
SGD-momentum & $M$ & $4$ \\
AdamW & $M, V$ & $8$ \\
\midrule
\textbf{AK-SGD (norm-key)} & $M$ & $\mathbf{4}$ \\
\textbf{AK-AdamW} & $M, V$ & $\mathbf{8}$ \\
\bottomrule
\end{tabular}
\end{center}

\begin{result}[Persistent overhead = 0]
The persistent optimizer-state memory overhead of normalized-key
AK-SGD against SGD-momentum, and of normalized-key AK-AdamW against
AdamW, is exactly zero.
\end{result}

\subsubsection{Saved Activations}

The standard backward already saves $X$ for $\nablaW$. The factored
delta also requires $X$ (to build $\Xhat$). Hence:

\begin{result}[Saved-activation overhead = 0]
No additional activation tensors must be saved between forward and
backward.
\end{result}

\subsubsection{Transient Memory During the Optimizer Step}

\textbf{Factored path.}
\begin{center}
\begin{tabular}{l c r}
\toprule
Tensor & Shape & Bytes (fp32) \\
\midrule
$\|x_i\|_2$ & $N$ & $4N$ \\
$\Xhat$ & $N \times n$ & $4Nn$ \\
$T = M\Xhat^\top$ & $m \times N$ & $4mN$ \\
$R = \nablaY^\top - T$ & $m \times N$ & in-place over $T$ \\
$U$ & $m \times n$ & accumulates into $M$ \\
\midrule
\textbf{Peak transient} & & $\mathbf{4Nn + 4mN + 4N}$ \\
\bottomrule
\end{tabular}
\end{center}

\textbf{Gram path.}
\begin{center}
\begin{tabular}{l c r}
\toprule
Tensor & Shape & Bytes (fp32) \\
\midrule
$\|x_i\|_2$ & $N$ & $4N$ \\
$\Xhat$ & $N \times n$ & $4Nn$ \\
$\Ghat$ & $n \times n$ & $4n^2$ \\
$C = M\Ghat$ & $m \times n$ & accumulates \\
\midrule
\textbf{Peak transient} & & $\mathbf{4Nn + 4n^2 + 4N}$ \\
\bottomrule
\end{tabular}
\end{center}

For tall/square layers with $N \gg n^2/m$, the Gram path is strictly
more memory-efficient. The cost from key normalization itself is the
$4Nn$ term, but if we permit \emph{in-place} replacement of $X$
by $\Xhat$ (the standard backward no longer needs $X$ once $\Xhat$ is
formed and $\nablaW$ has been computed for Gram layers), this term is
absorbed into the existing $X$ allocation.

\subsubsection{Per-Block Transient Peak Overhead}

Under fused backward+optimizer execution (per-layer immediate update,
no global $\nablaW$ retention across the network):
\begin{equation}\label{eq:peak-mem}
\boxed{\;\;
\mathrm{PeakMemOverhead}^{\text{norm-key}}_{\text{block}}
\le 4\,N \cdot n_{\max} \;\;\text{bytes (fp32)}
\;\;}
\end{equation}
where $n_{\max} = \gamma d$ for the FFN-down input.

\begin{result}[Concrete peak overhead]
For \(d=1024\), \(\gamma=8/3\), fp16 storage, and a microbatch of
\(N=8192\) tokens rather than the 524{,}288 tokens of a full optimizer step,
since only one microbatch of \(\hat X\) is live at a time, the largest
normalized-key transient allocation is
\begin{equation}
2 \cdot 8192 \cdot (8/3)\cdot 1024 \approx 44.7\text{ MB},\
\text{per block. }
\end{equation}
This is fully transient, since only one block's \(\hat X\) is live at a time.
\end{result}

\subsection{Theoretical Lower Bound}

\subsubsection{Sequential Read--Write Lower Bound}

The delta rule's intrinsic operation per sample comprises:
\begin{enumerate}
  \item \emph{Read:} $\hat v_t = M_t \xhat_t$, costing $2mn$ FLOPs.
  \item \emph{Write:} $M \mathrel{+}= \eta(v_t - \hat v_t)\xhat_t^\top$,
        costing $2mn$ FLOPs.
\end{enumerate}
Aggregating over $N$ samples gives a lower bound of $4mnN$ for batched
algorithms, or $2mnN$ for fully sequential ones, which is the through-$N$
order of \eqref{eq:dispatch}. The Gram order beats it whenever $n < 2m$
by moving the $N$ contraction into a matrix that several layers share.

\subsubsection{Tightness of the Asymptotic Result}

The dispatch rule together with cross-layer Gram sharing yields
\eqref{eq:block-overhead-deployed}, which is the theoretical minimum
under three constraints, GPU-batched execution, a standard transformer
architecture with no architectural relaxation, and independent block
computation with no cross-block sharing. The leading coefficient
$(6+4\gamma)$ saturates two structurally independent lower bounds at
once, the per-layer read and write bound above and the Gram-shared bound
from cross-layer input reuse. Further reduction would require relaxing
one of the three, for example low-rank or sparse keys, a custom fused
weight-gradient kernel, or cross-block input sharing, which the
architecture does not allow.

\section{Experimental Setup}
\label{app:setup}

\textbf{Architectures.}
Across all language modeling experiments we use a Llama-2-style~\citep{touvron2023llama} decoder-only
transformer with SwiGLU FFNs (expansion ratio 2.67), rotary position
embeddings (RoPE, base 10,000), pre-RMSNorm, no biases on linear layers, and
untied input/output embeddings. \Cref{tab:arch} lists the three scales used in
the main experiments. All three belong to the same width-only \(\mu\)P family:
\(d_{\mathrm{head}}=64\), \(n_{\mathrm{layers}}=24\) and the SwiGLU expansion
ratio are held fixed across scales so that scale enters only through
\(d_{\mathrm{model}}\). The 67M model is the proxy at which all
hyperparameters are tuned, and the 370M and 1B models are the \(\mu\)P transfer targets.

\begin{table}[t]
  \centering
  \caption{Model architectures. All three belong to the same width-only \(\mu\)P
    family, with $d_{\mathrm{head}}$, $n_{\mathrm{layers}}$, and the SwiGLU
    expansion ratio fixed across scales, so scale enters only through
    $d_{\mathrm{model}}$.}
  \label{tab:arch}
  \begin{tabular}{lcccccc}
    \toprule
    Scale & $d_{\mathrm{model}}$ & $n_{\mathrm{layers}}$ &
      $n_{\mathrm{heads}}$ & $d_{\mathrm{head}}$ & Vocab & Params \\
    \midrule
    67M (proxy)  & 384  & 24 & 6  & 64 & 32{,}000 & $\sim$67M \\
    370M (target) & 1024 & 24 & 16 & 64 & 32{,}000 & $\sim$370M \\
    1B (target) & 1792 & 24 & 28 & 64 & 32{,}000 & $\sim$1B \\
    \bottomrule
  \end{tabular}
\end{table}

\textbf{Data and training protocol.}
We pretrain on the 10BT subset of \textsc{FineWeb-Edu}~\citep{penedo2024fineweb}, tokenized with the
Llama-2 SentencePiece tokenizer (vocabulary size 32{,}000). Sequence length
is 2048 and the effective batch size is 256 sequences (524{,}288 tokens per
optimizer step), giving 19{,}073 optimization steps for one full pass over
the 10B-token budget. Training uses bfloat16 mixed precision, gradient norm
clipping at 1.0, a cosine learning-rate schedule with linear warmup over the
first 500 steps and decay to a floor of \(10^{-2}\) of the peak rate, and
no dropout. The validation set is a held-out 0.5\% slice of FineWeb-Edu (not
seen during training). The 1B run uses the 20BT budget, which is Chinchilla-optimal
at that size, with the cosine cycle stretched to the longer horizon and the warmup
held fixed at 500 steps, identically for both optimizers.

\textbf{The matched-level measurement window.}
Every loss gap, step reduction, and cost saving we report is measured over
matched validation-loss levels taken from step 2{,}000 onward, which is also the
range the figures plot. The window drops the warmup and the transient right after
it, where the two curves have not yet separated and a small difference in loss
maps to a very large difference in steps, so a metric read there says more about
the schedule than about the buffer. The window was fixed once, before any of the
numbers below were computed, and the same window is applied at every scale, to
every arm, and to both cost axes.

\textbf{Baselines and \DeltaMomentum instantiation.}
Our headline comparison is AK-AdamW against AdamW. AK-AdamW is obtained
by replacing AdamW's first-moment EMA accumulator with the
\DeltaMomentum update of \Cref{eq:deployed}, and the second moment, bias
correction, and decoupled weight decay are unchanged. The layer inputs
that the correction needs are captured per linear layer through forward
hooks, and the captured tensors are released immediately after the
optimizer step (Appendix~\ref{app:flops-memory}). The pure delta limit (\(\beta=1\)) is not
used in our main runs, and we always operate with \(\beta\in[0.9,1)\) so that
\DeltaMomentum admits a uniform decay floor along every direction.

\textbf{Hyperparameter selection via \(\mu\)P transfer.}
All hyperparameters are tuned on the 67M proxy and transferred to 370M and 1B
via the \(\mu\)P scaling rules for hidden weights:
\(\mathrm{lr}_{\mathrm{eff}} = \mathrm{lr}_{\mathrm{proxy}}\cdot
(d_{\mathrm{base}}/d_{\mathrm{model}})\) for both AdamW and AK-AdamW,
while the delta coefficient \(\eta\) and the EMA coefficient \(\beta_1\)
transfer directly without rescaling. The \(\eta\)-invariance follows from
the per-sample key normalization \(\|\hat{x}_t\|=1\)
(\Cref{thm:mup-eta} and \Cref{thm:unconditional-stability}), and is verified empirically
by a \(\mu\)P coordinate check at widths
\(\{128,256,512,1024, 2048\}\) (Appendix~\ref{app:mup-coord}).

We use Optuna~\citep{akiba2019optuna} with a TPE sampler and a
SuccessiveHalvingPruner to search over learning rate, \(\beta_1\),
\(\eta\) (AK-AdamW only), and the auxiliary learning rate for the embedding
and output-head group, with weight decay fixed at \(10^{-1}\) and
\(\beta_2=0.99\). The auxiliary rate is tuned as its own axis for AK-AdamW and for Muon,
because tying it to the main rate starves those parameters in both. AdamW uses one
rate for all parameter groups, which is the standard configuration for it and the
one its own \(\mu\)P rule expects, so it carries no separate auxiliary axis. Each trial trains the 67M proxy for 4{,}000 steps with the
LR scheduler period set to the full 19{,}073-step horizon so that trial
losses are comparable to early-training losses of the full run.
Appendix~\ref{app:hp-sweep} gives the final range of every axis, the pruner settings, the
budget accounting, and the boundary audit. The
configurations selected for the headline runs are listed in \Cref{tab:hps}.

\begin{table}[t]
  \centering
  \caption{Optuna-selected hyperparameters at the 67M proxy and their
    \(\mu\)P-transferred values at the two targets. The \(\eta\) coefficient and
    \(\beta_1\) do not rescale with width. ``aux lr'' is the auxiliary rate for the
    embedding and output-head group, tuned as its own axis.}
  \label{tab:hps}
  \small
  \begin{tabular}{lcccccccc}
    \toprule
    optimizer & lr@67M & lr@370M & lr@1B & aux lr@67M & $\beta_1$ & $\beta_2$ & $\eta$ & wd \\
    \midrule
    AdamW       & 4.0\,$\!\times\!10^{-3}$ & 1.5\,$\!\times\!10^{-3}$ & 8.6\,$\!\times\!10^{-4}$ & --- & 0.95 & 0.99 & --- & 0.1 \\
    AK-AdamW  & 3.0\,$\!\times\!10^{-4}$ & 1.13\,$\!\times\!10^{-4}$ & 6.4\,$\!\times\!10^{-5}$ & 1.24\,$\!\times\!10^{-2}$ & 0.99 & 0.99 & 0.37 & 0.1 \\
    Muon        & 8.0\,$\!\times\!10^{-3}$ & 4.90\,$\!\times\!10^{-3}$ & --- & 1.28\,$\!\times\!10^{-2}$ & 0.95 & --- & --- & 0.1 \\
    \bottomrule
  \end{tabular}
\end{table}

\textbf{Compute and seeds.}
All experiments are run on H200 GPUs (local cluster). On a single H200, a 67M training run takes roughly 13 hours with AdamW and 15 hours with AK-AdamW, a 370M run roughly 35 and 41 hours, and a 1B run roughly 112 and 132 hours. These figures are the cost of running each configuration and not a throughput comparison across scales, because the microbatch size is not the same at every scale and wider models run their matrix multiplications closer to the peak rate of the GPU. The like-for-like timing comparison is the training-loop step time of \Cref{app:realized-flops}, which is measured on a single H200 with the same microbatch size in both arms.
The 67M and 370M FineWeb-Edu comparisons are each repeated over 3 seeds (42, 777, 1234) for both AdamW and AK-AdamW, and we report per-seed bands at both scales. The Muon arm uses 3 seeds at each of 67M and 370M, with the same seed set (42, 777, 1234). The 1B comparison is single-seed because of compute, and we report it as a scale check rather than a significance claim. 

\section{Comparison with Structured Optimizers}
\label{app:structured}

This appendix reports the tuned Muon comparison in full and states why Shampoo and SOAP are absent. \Cref{sec:main} summarizes the finding.

\begin{table}[h]
  \centering
  \caption{AK-AdamW against a tuned Muon baseline on FineWeb-Edu, across matched validation-loss levels, in the same format as \Cref{tab:scale}. Loss gap is in nats and step reduction is the fraction of Muon steps saved to reach the same loss, both reported as mean and maximum over the matched levels with the seed standard deviation. Matched levels are taken over the same step-2{,}000 window as \Cref{tab:scale}.}
  \label{tab:muon}
  \small
  \begin{tabular}{lcccc}
    \toprule
    Scale (seeds) & Mean loss gap & Max loss gap & Mean step red. & Max step red. \\
    \midrule
    67M \ \ (3 seeds) & $0.120 \pm 0.008$ & $0.140 \pm 0.007$ & $49.4 \pm 2.4\%$ & $55.8 \pm 1.9\%$ \\
    370M (3 seeds) & $0.109 \pm 0.003$ & $0.132 \pm 0.003$ & $41.6 \pm 0.8\%$ & $48.1 \pm 0.7\%$ \\
    \bottomrule
  \end{tabular}
\end{table}

\begin{figure}[h]
  \centering
  \begin{subfigure}[t]{0.48\linewidth}
    \includegraphics[width=\linewidth]{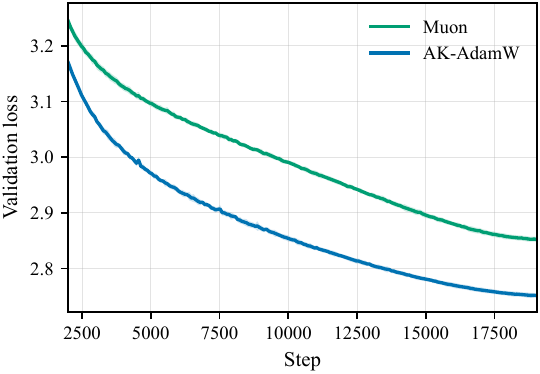}
    \caption{67M, 3 seeds.}
  \end{subfigure}\hfill
  \begin{subfigure}[t]{0.48\linewidth}
    \includegraphics[width=\linewidth]{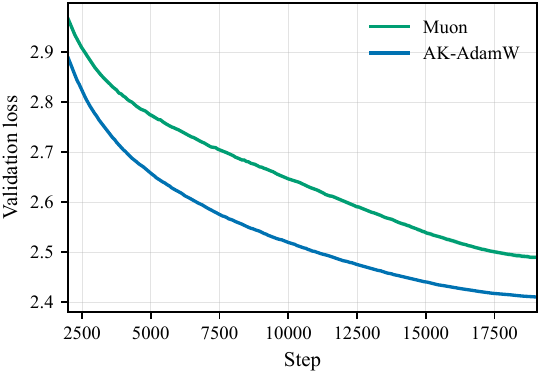}
    \caption{370M, 3 seeds.}
  \end{subfigure}
  \caption{AK-AdamW against a tuned Muon baseline on FineWeb-Edu, validation loss against optimizer step, seed means with per-seed bands. The sign of the gap is consistent in every run, and the two curves do not cross at either scale.}
  \label{fig:muon}
\end{figure}

\textbf{Why AdamW is the isolating baseline.}
The empirical contribution of this paper is the delta-rule update applied to the first-moment buffer, and isolating it requires a baseline that differs from AK-AdamW only in that update. AdamW is the unique such baseline. A head-to-head against a structured optimizer instead conflates the first-moment contribution with Newton--Schulz orthogonalization or with Kronecker-factored preconditioning, which act at a different level of the stack. That is why \Cref{sec:main} leads with AdamW and treats Muon as a reference point rather than as the isolating comparison, and why the Muon numbers are reported here rather than beside the headline result.

\textbf{Muon protocol.}
Muon was tuned under the same Optuna setup as the other two arms, with the same sampler, the same pruner, and the same trial horizon. Budgets were 33, 32, and 38 full-trial equivalents per tuned dimension for AdamW, AK-AdamW, and Muon, so Muon received the largest per-dimension budget of the three. Two further points make the tuning fair rather than merely equal in budget. Muon was tuned under its own $\mu$P rule, $\alpha = \Theta(n^{-1/2})$ with the $\sqrt{m/n}$ prefactor \citep{qiu2025hptransfer}, rather than under an Adam-family rule \citep{yang2021tuning}, since transferring an Adam rule to a spectrally normalized update is not a like-for-like transfer. And the auxiliary AdamW learning rate for the embedding and output-head group was decoupled and tuned as its own axis for Muon exactly as for AK-AdamW. The 370M Muon run transfers the 67M champion with no re-tuning under that same rule, which is the protocol the other two optimizers also follow.

\textbf{Scope of the Muon comparison.}
We make no general claim about Muon from \Cref{tab:muon}. We report it as a baseline tuned inside our protocol rather than as a ceiling for the method.

\textbf{Shampoo and SOAP.}
We did not run Shampoo or SOAP. Both keep a scalar-decay EMA first moment and add Kronecker-factored preconditioning on top of it, so the delta rule changes a different component of the stack and would slot into their buffers unchanged. The comparison that isolates our contribution is the EMA-against-delta swap at a fixed base optimizer, which is the argument of Appendix~\ref{app:related}, and we read Shampoo and SOAP as complementary rather than competing. The same holds for Muon, where the delta buffer would replace the EMA that feeds Newton--Schulz. \textsc{AK-Shampoo}, \textsc{AK-SOAP}, and \textsc{AK-Muon} are identified as future work in \Cref{sec:conclusion}.

\section{Mechanistic Diagnostics of Momentum Dynamics}
\label{app:mechanistic}

\subsection{Explanation of Momentum--Gradient Cosine Similarity}
\label{app:mom-grad-cos-sim}
The buffer--gradient cosine in stochastic optimization is \emph{not} expected to
approach $+1$, and is in general not even of definite sign. A Taylor expansion of
$g_t = \nabla\!\mathcal{L}(\theta_t)$ around $\theta_{t-1}$ together with the update rule
$\theta_t = \theta_{t-1} - \alpha\, M_{t-1}/(\sqrt{V_{t-1}}+\varepsilon)$ yields
\begin{equation}
\langle M_{t-1}, g_t\rangle
\;=\;
\langle M_{t-1}, g_{t-1}\rangle
\;-\;
\alpha\,\frac{\langle M_{t-1}, H_{t-1}\, M_{t-1}\rangle}{\sqrt{V_{t-1}}+\varepsilon}
\;+\;
\langle M_{t-1}, \xi_t\rangle,
\label{eq:rebound}
\end{equation}
where $H_{t-1}\succeq 0$ is the local Hessian and $\xi_t$ is stochastic gradient
noise. The middle term is non-negative and biases $\cos(M_{t-1}, g_t)$ downward
by an amount that depends on local curvature and step size, not on the buffer
rule. This deterministic ``stepping-rebound'' effect applies symmetrically to
EMA and \DeltaMomentum, since it is a property of stochastic descent itself rather than of
the optimizer. The diagnostic quantity is therefore the
\emph{gap between optimizers at matched training step}, and equivalently the
prediction error~$e_t$, which is monotone in
$1 - \cos(M_t, g_t)\,\cdot\,\|M_t\|_F/\|g_t\|_F$.

\section{Maximal-Update ($\mu$P) Scaling}
\label{app:muP}

We derive the width-scaling rule for the delta-rule coefficient $\eta$ in
the normalized variant under maximal update parameterization
\citep{yang2021tuning}, showing that $\eta$ is width-invariant and that the
base optimizer's $\mu$P learning-rate rule is inherited unchanged. The
analysis is presented for a single hidden linear layer, and the multi-layer
extension follows by the same per-layer reasoning, since the delta-rule
correction is layer-local.

\subsection{Setup and Per-Coordinate Scales}
Consider a hidden linear layer $W \in \R^{m\times n}$ at width $n$ with
$m = \Theta(n)$. Under $\mu$P, the input $x \in \R^n$ has per-coordinate
scale $\Theta(1)$, so $\|x\|^2 = \Theta(n)$ and the unit key
satisfies $\|\xhat\|^2 = 1$ with per-coordinate scale $\Theta(n^{-1/2})$.
The output-side error $\delta \in \R^m$ has per-coordinate scale
$\Theta(1)$ for hidden layers under $\mu$P. The gradient
$g = \delta x^\top$ has per-coordinate scale $\Theta(1)$, and the EMA
momentum buffer $M^{\mathrm{EMA}} = \E[g]$ has per-coordinate scale
$\Theta(n^{-1/2})$ under the random-sign cancellation of \Cref{cor:mup-base-lr}.

\textbf{Target scale for the \DeltaMomentum buffer.} The normalized-key
rule~\eqref{eq:dm-norm-update} uses the normalized key $\xhat$ in both terms. The natural injection signal is
$\eta\delta_t \xhat_t^\top$, which has per-coordinate scale
$\Theta(\eta) \cdot \Theta(1) \cdot \Theta(n^{-1/2}) = \Theta(\eta/\sqrt{n})$.
Setting $\eta = \Theta(1)$ gives a buffer with per-coordinate scale
$\Theta(n^{-1/2})$. We show below that this is the consistent fixed-point
scale under the recurrence~\eqref{eq:dm-norm-update} and that this scale
matches the postprocessing requirements of each base optimizer, recovering
the standard $\mu$P learning-rate rules.

\subsection{Width Invariance of $\eta$}

\begin{theorem}[$\mu$P width invariance of $\eta$]
\label{thm:mup-eta}
Under $\mu$P with input $x$ of per-coordinate scale $\Theta(1)$ and output
error $\delta$ of per-coordinate scale $\Theta(1)$, the normalized
\DeltaMomentum recurrence
\begin{equation}
M_t = \beta M_{t-1} + \eta(\delta_t - M_{t-1}\xhat_t)\xhat_t^\top
\label{eq:dm-mup}
\end{equation}
admits a fixed-point per-coordinate scale of $\Theta(n^{-1/2})$ if and only
if the delta-rule coefficient is width-invariant: $\eta = \Theta(1)$.
\end{theorem}

\begin{proof}
We track per-coordinate scales of $M_t$ in steady state. Suppose
$M_{t-1}$ has per-coordinate scale $s$, i.e.\ $|M_{t-1,ij}| = \Theta(s)$. We
compute the per-coordinate scale of each term in~\eqref{eq:dm-mup}.

\emph{Decay term.} $\beta M_{t-1}$ has per-coordinate scale $\Theta(s)$.

\emph{Injection term.} $\eta\,\delta_t\xhat_t^\top$ has per-coordinate scale
$\Theta(\eta) \cdot \Theta(1) \cdot \Theta(n^{-1/2}) = \Theta(\eta/\sqrt{n})$.

\emph{Retrieval term.} The vector $M_{t-1}\xhat_t \in \R^m$ has $i$-th
coordinate $\sum_{j=1}^n M_{t-1,ij}\xhat_{t,j}$, a sum of $n$ terms each of
order $\Theta(s) \cdot \Theta(n^{-1/2}) = \Theta(s/\sqrt{n})$. Under random
sign cancellation (the standard $\mu$P fluctuation argument applied to the
inner product between a buffer row and the unit key), the magnitude of this
sum is $\Theta(\sqrt{n} \cdot s/\sqrt{n}) = \Theta(s)$. Multiplying by
$\eta\xhat_t^\top$ then gives per-coordinate scale
$\Theta(\eta) \cdot \Theta(s) \cdot \Theta(n^{-1/2}) = \Theta(\eta s/\sqrt{n})$.

\emph{Fixed-point balance.} At steady state $|M_t| = |M_{t-1}| = \Theta(s)$,
so the largest term on the right-hand side of~\eqref{eq:dm-mup} must scale
as $\Theta(s)$. Since the decay term is $\Theta(s)$ and the retrieval term
is $\Theta(\eta s/\sqrt{n}) = o(s)$ for any $\eta = O(1)$, the binding
constraint comes from the injection term: we need
$\Theta(\eta/\sqrt{n}) = \Theta(s)$, giving $s = \Theta(\eta/\sqrt{n})$.
For the resulting $M_t$ to have a width-invariant per-coordinate
\emph{ratio} relative to its expectation, which is the standard $\mu$P feature-learning
condition, we require $\eta = \Theta(1)$, hence $s = \Theta(n^{-1/2})$.

Conversely, if $\eta = \Theta(n^{-c})$ for any $c > 0$, the buffer scale
becomes $\Theta(n^{-1/2 - c})$, which violates feature learning. If
$\eta = \omega(1)$, the retrieval term diverges relative to the decay term
and stability fails. Hence $\eta = \Theta(1)$ is necessary and sufficient.
\end{proof}

\textbf{The deployed rule.} The deployed rule \eqref{eq:deployed} has the forgetting operator of \eqref{eq:dm-mup}, and its gradient term is the weight gradient of the backward pass. Its fixed point $\Mstar$ is \eqref{eq:dm-fixedpoint-norm} with $\Ghat$ replaced by $\Gbar$ (\Cref{sec:gradient-estimator}), so $\Mstar u_i = \Gbar u_i/(\mu + \lamhat_i)$ for each eigenvector $u_i$ of $\Sighat$, where $\Gbar$ is the fixed point of the EMA buffer, and since $0 \le \lamhat_i \le 1$ (\Cref{prop:bounded-sigma-hat}) the factor $1/(\mu + \lamhat_i)$ lies between $1/(1+\mu)$ and $1/\mu$. Stability requires $\eta < 1 + \beta$ (Appendix~\ref{app:stability_norm_key_dm}), so $\mu = (1-\beta)/\eta$ is bounded below, and the deployed buffer has the per-coordinate scale $\Theta(n^{-1/2})$ of the EMA buffer if and only if $\eta$ does not vanish with width, that is $\eta = \Theta(1)$.

\subsection{Inherited Base-optimizer Learning-Rate Rules}

\begin{corollary}[Inherited $\mu$P learning-rate rules]
\label{cor:mup-base-lr}
With $\eta = \Theta(1)$, the normalized \DeltaMomentum buffer $M_t$ has
per-coordinate scale $\Theta(n^{-1/2})$, identical to the per-coordinate
scale of the EMA momentum buffer (because EMA's buffer estimates the
\emph{population gradient} $\Gbar = \E[\delta x^\top]$, whose per-coordinate
expectation has scale $\Theta(n^{-1/2})$ under $\mu$P with random-sign
cancellation, and the per-sample buffer scale is $\Theta(1)$ while its expectation
exhibits the same $\Theta(n^{-1/2})$ shrinkage). Consequently, every
postprocessing operation on $M_t$ inherits the same scaling as it would
under EMA, and the base-optimizer's standard $\mu$P learning-rate rule
applies unchanged:
\begin{itemize}
\item \emph{AK-SGD:} $\alpha = \Theta(1)$ for hidden layers, $\Theta(1/n)$
for output heads (matching SGD-with-momentum under $\mu$P).
\item \emph{AK-AdamW:} the preconditioned update $M_t/\sqrt{v_t + \varepsilon}$
has per-coordinate scale $\Theta(1)$ (because both $M_t$ and $\sqrt{v_t}$
scale as $\Theta(n^{-1/2})$), so $\alpha = \Theta(1/n)$ for hidden
layers, matching AdamW under $\mu$P.
\item \emph{AK-Muon:} the Newton--Schulz output $\NS(M_t)$ has unit
singular values and per-entry scale $\Theta(n^{-1/2})$ for an $m \times n$
matrix with $m = \Theta(n)$, so $\alpha = \Theta(n^{-1/2})$ with the
$\sqrt{m/n}$ prefactor of Muon, matching Muon under $\mu$P.
\end{itemize}
\end{corollary}

\begin{proof}
Per-coordinate scale of $M_t$ is $\Theta(n^{-1/2})$ by
\Cref{thm:mup-eta}. The Adam second moment $v_t$ averages $g_t^{\odot 2}$
for the weight gradient $g_t$ of each step, whose entries are averages over
the $N$ tokens of the step. Since $N$ is far larger than the width $n$,
these entries have the per-coordinate scale $\Theta(n^{-1/2})$ of $\Gbar$.
Hence $\sqrt{v_t} = \Theta(n^{-1/2})$ and $M_t/\sqrt{v_t} = \Theta(1)$. With this
unit-scale preconditioned update, AdamW's standard $\mu$P learning-rate
rule $\alpha = \Theta(1/n)$ for hidden weights gives
$\alpha M_t/\sqrt{v_t} = \Theta(1/n)$ per coordinate, i.e.\ feature
learning. Newton--Schulz produces an orthogonal-like matrix with unit
singular values, and for an $m \times n$ matrix with $m,n = \Theta(n)$ this
forces per-entry scale $\Theta(n^{-1/2})$, matching the standard Muon
analysis. SGD without preconditioning uses $\alpha = \Theta(1)$ on the
buffer of scale $\Theta(n^{-1/2})$, giving $\alpha M_t = \Theta(n^{-1/2})$
per coordinate, the standard $\mu$P feature-learning scale.
\end{proof}

\begin{remark}[$\mu$Transfer]
\label{rem:mu-transfer}
\Cref{thm:mup-eta} implies that $\eta$, once tuned at any width, transfers
without rescaling to any other width. Combined with
\Cref{cor:mup-base-lr}, this means the \DeltaMomentum family is
\emph{$\mu$P-transparent}: it adds exactly one additional hyperparameter
($\eta$) over the base optimizer, and that hyperparameter requires no
width-dependent scaling. We empirically verify width invariance of $\eta$
via coordinate checks in Appendix~\ref{app:mup-coord}.
\end{remark}

\begin{remark}[Why normalization matters for $\mu$P]
\label{rem:mup-normalization}
The width-invariance of $\eta$ rests on the unit norm $\|\xhat\| = 1$,
which makes the spectral norm of the rank-one operator $\xhat_t\xhat_t^\top$
equal to 1 independently of $n$. For the unnormalized variant, $\|x\|^2 = \Theta(n)$,
and the corresponding operator $x_t x_t^\top$ has spectral norm $\Theta(n)$,
to maintain stability we would require $\eta = \Theta(1/n)$, which alters
the steady-state scaling of $M_t$ and the inherited base-optimizer
learning-rate rule. The normalized variant avoids these complications
entirely, and all $\mu$P scalings of the base optimizer carry over verbatim.
\end{remark}

\section{Generality Across Optimizer Families and Architectures on CIFAR-10}
\label{app:cifar-10}

This appendix complements \Cref{sec:main-deltaSGD} with full setup details and the corresponding test-accuracy comparison. We compare AK-SGD with momentum against vanilla SGD with momentum on a 3-hidden-layer MLP (4096, 4096, 2048 units per each layer, respectively, with ReLU activations) trained on CIFAR-10~\citep{krizhevsky2009learning} for 120 epochs at batch size 256. We manually sweep the learning rate independently for each optimizer over ($\alpha \in \{5\!\times\!10^{-2},  1\!\times\!10^{-1}, 3\!\times\!10^{-1}\}$ for SGD and $\alpha \in \{0.9,  1.3, 1.7, 2.1\}$ for AK-SGD) and report the best configuration. The basin of $\alpha$ for SGD was found in $\alpha=0.1$ as achieving the minimal loss, whereas for AK-SGD, $\alpha=1.7$ was the identified learning rate achieving minimal loss, and the momentum coefficient $\beta$ is swept within range $\{0.9,  0.95, 0.99\}$ for both optimizers, where $\beta=0.9$ was preferable for SGD and $\beta=0.99$ for AK-SGD. For AK-SGD we additionally search $\eta \in \{0.1, 0.5, 0.9, 0.99\}$ and find the optimal point in $\eta=0.99$. \Cref{fig:cifar} shows the training and validation trajectories, and \Cref{tab:cifar} reports final test accuracy.

\begin{figure}[h]
  \centering
  \begin{subfigure}[t]{0.49\linewidth}
    \includegraphics[width=\linewidth]{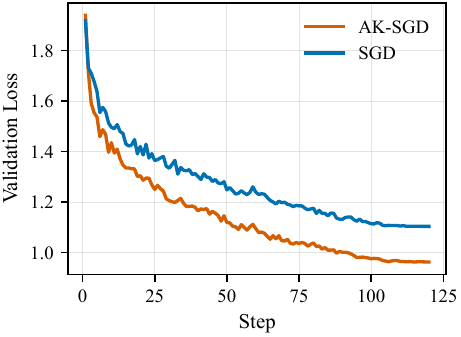}
    \caption{Validation loss vs.\ epoch. The horizontal axis is the epoch index.}
    \label{fig:cifar-loss}
  \end{subfigure}\hfill
  \begin{subfigure}[t]{0.49\linewidth}
    \includegraphics[width=\linewidth]{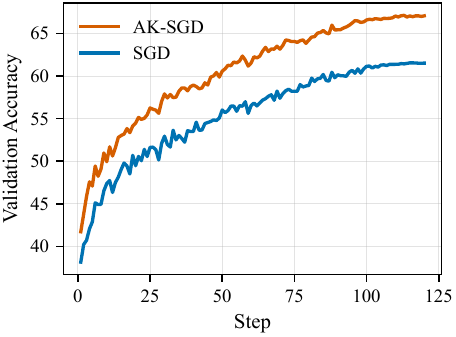}
    \caption{Validation accuracy vs.\ epoch. The horizontal axis is the epoch index.}
    \label{fig:cifar-acc}
  \end{subfigure}
  \caption{AK-SGD vs.\ SGD-momentum on CIFAR-10 with a 3-layer MLP. AK-SGD matches or exceeds SGD throughout training and achieves a higher final test accuracy, demonstrating that the delta-rule mechanism transfers beyond Adam's diagonal-preconditioner setting.}
  \label{fig:cifar}
\end{figure}

\begin{table}[h]
  \centering
  \caption{Final test accuracy on CIFAR-10 (3-hidden-layer MLP, 120 epochs). Best per-optimizer learning rate selected from an independent sweep.}
  \label{tab:cifar}
  \begin{tabular}{lcc}
    \toprule
    Optimizer & Final test acc.\ (\%) & $\Delta$ vs.\ SGD \\
    \midrule
    SGD-momentum & {61.51} & --- \\
    AK-SGD     & {67.06} & $+${5.55} \\
    \bottomrule
  \end{tabular}
\end{table}

The qualitative pattern matches the language-modeling result, since the gain appears early, is sustained throughout training, and translates into a positive final-metric difference. This experiment is a controlled test of whether the delta-rule mechanism transfers across optimizer families, run on a 3-layer MLP for that purpose, and the answer is positive.

\subsection{Standard Vision Architectures}
\label{app:cifar-vision}

This subsection complements \Cref{sec:main-deltaSGD} and \Cref{tab:vision}. We train ResNet-18 and ViT-Tiny on CIFAR-10 for 120 epochs, AK-AdamW against AdamW, with 3 seeds per optimizer and an identical log-scale learning-rate sweep for each. All other settings are held at the values that the architecture's usual recipe prescribes and are the same across the two arms.

\begin{figure}[h]
  \centering
  \begin{subfigure}[t]{0.48\linewidth}
    \includegraphics[width=\linewidth]{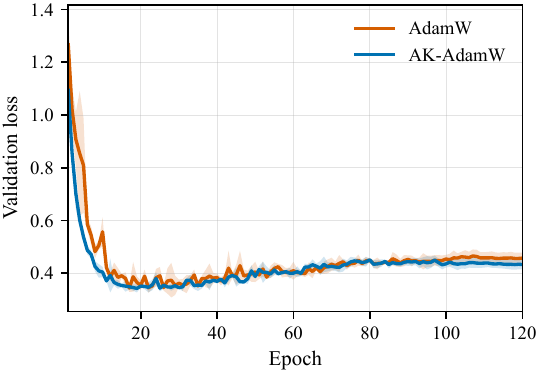}
    \caption{ResNet-18, validation loss.}
  \end{subfigure}\hfill
  \begin{subfigure}[t]{0.48\linewidth}
    \includegraphics[width=\linewidth]{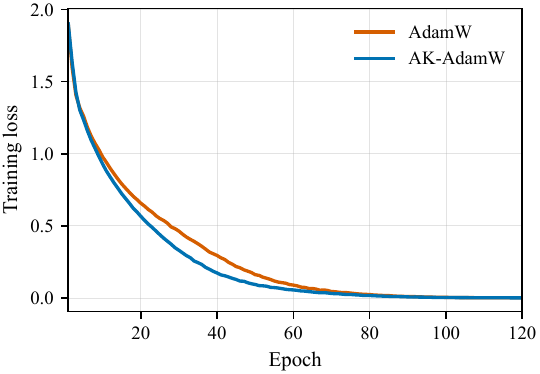}
    \caption{ViT-Tiny, training loss.}
  \end{subfigure}
  \caption{AK-AdamW against AdamW on CIFAR-10 over 120 epochs, 3 seeds, seed means with per-seed bands. We read ResNet-18 on validation loss and ViT-Tiny on training loss, for the reason given below.}
  \label{fig:cifar-vision}
\end{figure}

We read the two architectures on different curves, and the reason is what each curve can measure. ResNet-18 carries the locality and weight-sharing priors that CIFAR-10 rewards, so it generalizes well at this dataset size and its validation curve tracks optimization progress rather than a generalization gap. That makes validation loss both meaningful and the stricter of the two readings, so we use it. ViT-Tiny has no such inductive bias and is trained here from scratch without the large-scale pretraining or heavy augmentation that vision transformers normally rely on, so on a dataset of this size it fits the training set long before its validation loss reflects anything about the optimizer. Its validation curve is then dominated by the generalization gap of the architecture, which is a property of the model and not of the buffer rule, so we read ViT-Tiny on training loss, where the optimization question the paper asks is the one being measured. The complementary pair of curves is in \Cref{fig:cifar-vision-alt}. 
The ResNet-18 training loss has the same ordering, and the ViT-Tiny validation loss of AK-AdamW is lower than that of AdamW only until about epoch 28.

\begin{figure}[h]
  \centering
  \begin{subfigure}[t]{0.48\linewidth}
    \includegraphics[width=\linewidth]{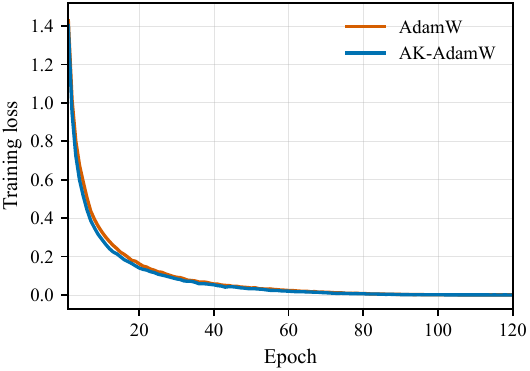}
    \caption{ResNet-18, training loss.}
  \end{subfigure}\hfill
  \begin{subfigure}[t]{0.48\linewidth}
    \includegraphics[width=\linewidth]{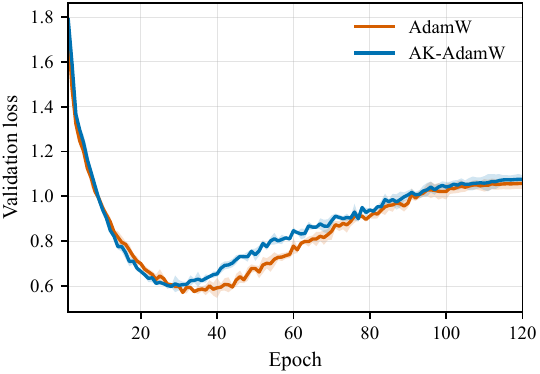}
    \caption{ViT-Tiny, validation loss.}
  \end{subfigure}
  \caption{The complementary pair to \Cref{fig:cifar-vision}, reported for completeness. 
  The ResNet-18 training loss has the same sign of the gap as in \Cref{fig:cifar-vision}. 
  The ViT-Tiny validation loss of AK-AdamW is lower than that of AdamW until about epoch 28 and higher after about epoch 30.
  Both ViT-Tiny validation curves rise after their minima near epoch 30 while the training losses in \Cref{fig:cifar-vision} keep decreasing.
  This rise is the generalization gap of the architecture on a dataset of this size rather than a difference between the optimizers.}
  \label{fig:cifar-vision-alt}
\end{figure}

We report the gap in training loss at matched epochs rather than a final metric, because both architectures saturate on CIFAR-10. Training accuracy reaches $99.99\%$ and test accuracy sits near $84\%$ for ViT-Tiny and $94\%$ for ResNet-18 under either optimizer, so final performance cannot separate the two and convergence speed is the quantity that can be measured. We average over the phase in which the curves separate, which is roughly the first 60 epochs for ViT-Tiny and the first 20 for ResNet-18, and we report the gap both in nats and relative to the AdamW loss at the same epoch. The sign of the gap is the same in all runs. On validation loss the ordering holds for ResNet-18 over its averaging window, and for ViT-Tiny only until about epoch 28, after which both validation losses rise and that of AK-AdamW is higher. 
This supports that the mechanism is not specific to language modeling, while language-model pretraining remains where the claims of this paper are made.

\section{$\mu$P Coordinate Check}
\label{app:mup-coord}

\Cref{thm:mup-eta} (Appendix~\ref{app:muP}) predicts that the delta-rule coefficient $\eta$ is width-invariant under maximal-update parameterization. We verify this empirically with the standard $\mu$P coordinate-check protocol, training the same Llama-2-style architecture at five widths $d_{\mathrm{model}} \in \{128, 256, 512, 1024, 2048\}$ for 20 optimization steps across five random seeds, holding all hyperparameters fixed except for the width-variant hyperparameters such as learning rate. In particular $\eta$ and $\beta_1$ are transferred without rescaling, at $\eta = 0.5$ and $\beta_1 = 0.99$. The specific value of $\eta$ is immaterial to this test, which asks only whether a single value transfers across widths, and it sits in the flat basin of Appendix~\ref{app:eta-sweep}. 
We measured two per-coordinate scales whose width-invariance certifies $\mu$P-cleanness:
\begin{itemize}
\item \textbf{Activation RMS}: the root-mean-square per-coordinate magnitude of each layer's pre-activation, $\mathrm{RMS}(y_t) = (\tfrac{1}{m}\sum_i y_{t,i}^2)^{1/2}$, which should remain $\Theta(1)$ across widths.
\item \textbf{Feature update RMS}: the per-coordinate magnitude of the change the realized weight update makes to the layer output, $\mathrm{RMS}(\Delta W x)$ with $\Delta W = W_{\mathrm{after}} - W_{\mathrm{before}}$, which should remain $\Theta(1)$ across widths under AK-AdamW's inherited $\mu$P learning-rate rule (\Cref{cor:mup-base-lr}).
\end{itemize}

Both quantities are measured in feature space, meaning projected onto the layer
input, and that choice is deliberate. The parameter-space counterparts
$\mathrm{RMS}(\Delta\theta)$ and $\mathrm{RMS}(\alpha \nabla W)$ are $\Theta(1/n)$
under correct $\mu$P and shrink with width, so a collapse test on them certifies
nothing. The gradient is also the wrong proxy for the update here, since the
second moment of Adam rescales it before it reaches the weights, so
$\alpha \nabla W \neq \Delta W$ for the optimizers we compare. Taking the
realized update and projecting it onto the input recovers the quantity the theory
actually constrains. Under correct $\mu$P both scales are $\Theta(1)$ in width, so the curves for
different widths should lie on one another, and a width-dependent trend in either
one would open a spread that grows with $n$. We read the check by that collapse.

\begin{figure}[h]
  \centering
  \begin{subfigure}[t]{0.49\linewidth}
    \includegraphics[width=\linewidth]{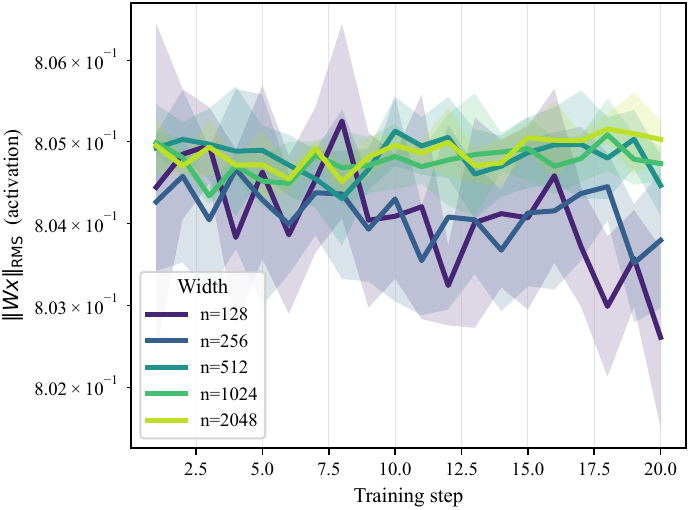}
    \caption{Activation RMS vs.\ training step, by width.}
    \label{fig:mup-act-rms}
  \end{subfigure}\hfill
  \begin{subfigure}[t]{0.49\linewidth}
    \includegraphics[width=\linewidth]{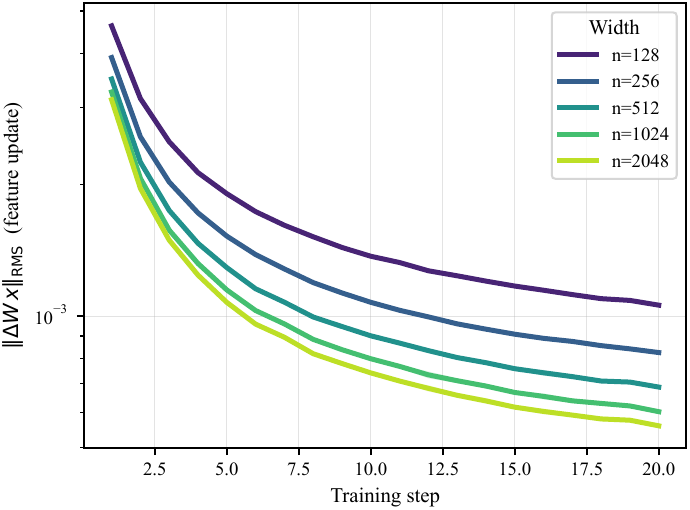}
    \caption{Feature update RMS vs.\ training step, by width.}
    \label{fig:mup-update-rms}
  \end{subfigure}
  \caption{$\mu$P coordinate check for AK-AdamW at widths $d_{\mathrm{model}} \in \{128, 256, 512, 1024, 2048\}$ with $\eta$ held fixed across widths. Activation RMS lies on one curve at every width. Feature-update RMS decreases with width by a factor that shrinks with each doubling, so the curves converge as width grows, which supports that $\eta$ requires no width-dependent rescaling and that the inherited $\mu$P scaling rule applies unchanged.}
  \label{fig:mup-coord}
\end{figure}

The convergence of the per-coordinate-scale curves as width grows in \Cref{fig:mup-coord} supports the theoretical prediction that a single $\eta$ tuned at one width transfers without modification to other widths. This is the practical content of \Cref{thm:mup-eta}.

\section{Hyperparameter Sweep}
\label{app:hp-sweep}

We use Optuna~\citep{akiba2019optuna} with a TPE sampler and a SuccessiveHalvingPruner for hyperparameter selection on the 67M proxy. The search space, search budget, and selection criterion are summarized below.

\textbf{Search space.}
We run separate Optuna studies for AdamW, AK-AdamW, and Muon with single seed (42).
We report the final range of each axis, which is the range the reported optimum
was selected inside. Ranges were refit whenever an optimum approached a boundary,
and each prior optimum was re-enqueued as an anchor trial before sampling resumed.
For AdamW, $\alpha \in [2{\times}10^{-3},\, 8{\times}10^{-3}]$ on log scale and
$\beta_1 \in \{0.9, 0.95\}$. For AK-AdamW,
$\alpha \in [5{\times}10^{-5},\, 4.5{\times}10^{-4}]$ on log scale,
$\beta_1 \in \{0.95, 0.99\}$, and $\eta$ TPE-sampled from $[0.05,\, 1.0]$ on log
scale. For Muon, $\alpha \in [3{\times}10^{-3},\, 3{\times}10^{-2}]$ on log scale
and $\beta_1 \in \{0.9, 0.95\}$. For AK-AdamW
and for Muon we additionally search the auxiliary AdamW learning rate for the
embedding and output-head group as its own axis rather than tying it to the main
rate, since tying it starves those parameters, over
$[2{\times}10^{-3},\, 4{\times}10^{-2}]$ for AK-AdamW and
$[2{\times}10^{-3},\, 5{\times}10^{-2}]$ for Muon. All studies
hold $\beta_2 = 0.99$ and weight decay $10^{-1}$ fixed. Each trial trains the
67M proxy for up to 4{,}000 steps, evaluating every 80 steps, with the cosine LR
scheduler period set to the full 19{,}073-step horizon so that trial losses are
comparable to early-training losses of the headline run. The
SuccessiveHalvingPruner uses \texttt{min\_resource}~$=600$,
\texttt{reduction\_factor}~$=3$, and \texttt{min\_early\_stopping\_rate}~$=1$,
which places exactly one rung inside the trial budget, realized at step $1{,}840$
of $4{,}000$. Champions are selected from full-budget completions only and never
from rung values, because rung rank and final rank are close to uncorrelated
among survivors in our studies.

\textbf{Sweep budget.}
We count budget in full-trial equivalents, where one equivalent is one trial run
to the full 4{,}000 steps, summing the steps every trial actually consumed. The
AdamW studies ran $101$ trials with $37$ completions for $66.3$ equivalents over
$2$ tuned dimensions, the AK-AdamW studies ran $207$ trials with $72$
completions for $129.2$ equivalents over $4$ dimensions, and the Muon studies ran
$249$ trials with $62$ completions for $115.3$ equivalents over $3$ dimensions.
Normalized by the number of tuned dimensions, the three arms received $33$, $32$,
and $38$ full-trial equivalents per dimension for AdamW, AK-AdamW, and Muon
respectively. Muon therefore received the largest per-dimension budget of the
three and our own method the smallest, which is the direction that matters for
reading the comparison. Roughly $65$ of the Muon trials are zero-cost prunes from
a co-scheduling out-of-memory incident and consumed essentially no steps, leaving
about $184$ informative trials, and they are excluded from the equivalents above
rather than counted as search effort. Most remaining trials are pruned, which is
by design, since TPE concentrates samples near the basin and successive halving
halts the rest at the rung. Total consumed compute is not equal across arms,
since dimensionality and the number of refit rounds differ, so we claim identical
protocol and per-dimension parity rather than equal GPU-hours. The Optuna SQLite
study databases are released with the code.

\textbf{Selection criterion.} The objective is the validation loss at the end of each trial. The configuration with the lowest validation loss is selected for the headline run, and no test-set information is used in selection.

\textbf{Selected configurations.}
The selected hyperparameters at 67M and the corresponding $\mu$P-transferred
values at the two targets are reported in \Cref{tab:hps}. AdamW selects
$\alpha = 4.0{\times}10^{-3}$ and $\beta_1 = 0.95$. 
AK-AdamW selects
$\alpha = 3.0{\times}10^{-4}$, $\beta_1 = 0.99$, and $\eta = 0.37$. 
Every completed AK-AdamW trial used $\beta_1 = 0.99$, which is larger
than AdamW prefers, and the recipe at the end of this appendix explains why the
per-direction memory coefficient makes that the expected choice. The selected $\eta$ sits
inside the flat basin characterized in Appendix~\ref{app:eta-sweep}, where the final loss
moves by at most $0.015$ nats across $\eta \in [0.10,\, 0.99]$, so neighbouring
values in that basin do not change the headline configuration.

The AK-AdamW peak learning rate is more than ten times smaller than
AdamW's, and the fixed point of the deployed rule, \eqref{eq:dm-fixedpoint-norm} with
$\Ghat$ replaced by $\Gbar$ (\Cref{sec:gradient-estimator}), accounts for it.
With the second moment unchanged, the fixed point of the buffer along $u_i$ is
$1/(\mu + \lamhat_i)$ times that of the EMA buffer, a factor between $0.97$ and
$37$ at the selected $\beta_1 = 0.99$ and $\eta = 0.37$.

Every reported optimum is interior to its final range, with completed trials on
both sides of it, so no reported value rests on a boundary. Writing the position
of each optimum as a fraction of its final range on the log axis, where $0$ is the
lower bound and $1$ the upper, AdamW $\alpha$ sits at $0.51$, AK-AdamW
$(\alpha, \alpha_{\mathrm{aux}}, \eta)$ at $(0.82, 0.61, 0.67)$, and Muon
$(\alpha, \alpha_{\mathrm{aux}})$ at $(0.43, 0.58)$. The horizontal bars in
\Cref{fig:optuna} draw these ranges. 
The categorical $\beta_1$ choices of AdamW and AK-AdamW both rest on completed trials rather than on defaults, including a controlled probe that runs one of the two optimizers at the $\beta_1$ value the other adopted.

\begin{figure}[h]
  \centering
  \begin{subfigure}[t]{0.325\linewidth}
    \includegraphics[width=\linewidth]{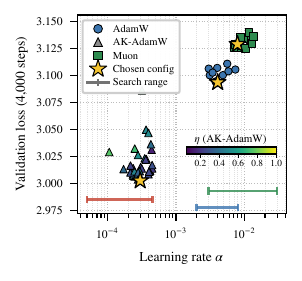}
    \caption{Trial loss vs.\ learning rate $\alpha$.}
    \label{fig:optuna-lr}
  \end{subfigure}\hfill
  \begin{subfigure}[t]{0.325\linewidth}
    \includegraphics[width=\linewidth]{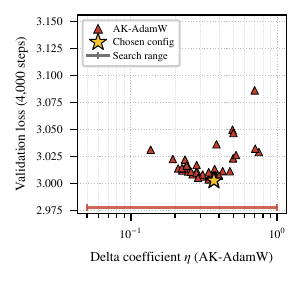}
    \caption{Trial loss vs.\ delta coefficient $\eta$.}
    \label{fig:optuna-eta}
  \end{subfigure}\hfill
  \begin{subfigure}[t]{0.325\linewidth}
    \includegraphics[width=\linewidth]{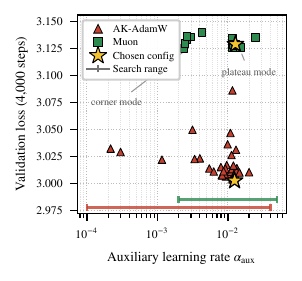}
    \caption{Trial loss vs.\ auxiliary rate $\alpha_{\mathrm{aux}}$.}
    \label{fig:optuna-aux}
  \end{subfigure}
  \caption{Optuna trial losses on the 67M proxy, with the chosen configuration of each arm starred and the final range of each axis (Appendix~T) drawn as a horizontal bar. 
  (a) The three arms occupy well-separated learning-rate basins, and each optimum is interior to its own final range with completed trials on both sides. (b) The $\eta$ surface is flat over most of a decade, which is the trial-level counterpart of the sensitivity ablation in Appendix~\ref{app:eta-sweep}. (c) The auxiliary rate for the embedding and output-head group separates into two modes. The corner mode collects trials that pin the auxiliary rate near the low end of the interval and starves those parameters, and the plateau mode collects the trials that do not, which is why we tune this rate as its own axis rather than tying it to the main rate.}
  \label{fig:optuna}
\end{figure}

\textbf{Recommended tuning recipe.}
The marginal tuning cost over the EMA momentum is one coefficient. Based on the
studies above, we recommend setting $\beta = 0.99$ and $\eta = 0.4$, keeping
$\beta_2$ and weight decay at the AdamW settings, and setting the learning rate
near one tenth of the tuned AdamW value and tuning only that.
Three facts make this cost-efficient.
First, $\beta = 0.99$ is larger than AdamW prefers, because the per-direction
memory coefficient is $\tbeta_i$ rather than $\beta$, so every direction with
positive input density forgets faster than nominal and a larger nominal $\beta$
recovers the same effective horizon.
Second, $\eta$ is easy to select, since the final loss moves by at most $0.015$
nats across $\eta \in [0.10, 0.99]$ (Appendix~\ref{app:eta-sweep}) and $\eta$
enters \Cref{thm:fixedpoint} only through the ridge $\mu$, where the fixed point
interpolates between the plain and the whitened gradient. Anything in
$[0.3, 0.5]$ is a safe default. Normalized keys also remove any data-dependent
constraint, since $\tr\Sighat = 1$ places $\tbeta_i$ inside $[\beta - \eta, \beta]$
for any data, so $\eta \le \beta$ keeps every coefficient non-negative and
stability needs only $\eta < 1 + \beta$.
Third, the learning rate does not carry over from AdamW because the buffer is a
Wiener predictor rather than an average of gradients.
Both $\eta$ and $\beta$ are width-invariant under $\mu$P (\Cref{thm:mup-eta}) and
the learning rate is transferred by the standard $\mu$P rule, so the whole recipe
is tuned once at the smallest proxy.

\section{$\eta$ Sensitivity Ablation}
\label{app:eta-sweep}

\Cref{thm:fixedpoint} relates the regularization parameter $\mu = (1-\beta)/\eta$ to the fixed-point structure of the buffer, and this in turn predicts that final loss should depend smoothly on $\eta$, with $\eta \to \beta$ approaching the unregularized (pure delta-rule) Wiener fixed point. 
To verify this prediction and empirically characterize the role of $\eta$, we sweep $\eta$ on the 67M proxy at the otherwise hyperparameters-fixed AK-AdamW configuration of \Cref{tab:hps}. 
Each run uses batch size 256 and context length 2048 (524K tokens/step), trained for 5,440 steps ($\approx$ 2.85B tokens, $\approx$ 2.12$\times$ Chinchilla optimal for this model size).

\begin{table}[h]
  \centering
  \caption{Final validation loss on FineWeb-Edu (2.85BT, 67M proxy) as a function of the delta-rule coefficient $\eta$, with all other AK-AdamW hyperparameters fixed at the values in \Cref{tab:hps}. 
  The configuration selected by Optuna ($\eta = 0.37$) was chosen independently of this ablation and lies between the $0.30$ and $0.50$ grid points.}
  \label{tab:eta-sweep}
  \begin{tabular}{lcccccccc}
    \toprule
    $\eta$        & 0.01 & 0.10 & 0.30 & 0.50 & 0.70 & 0.99 \\
    \midrule
    Final val.\ loss & 
    {3.063} & 
    {3.029} & {3.034} & {3.035} & {3.035} & 
    {3.020} \\
    \bottomrule
  \end{tabular}
\end{table}

The dependence on $\eta$ is relatively flat for $\eta \geq 0.1$, with a wide basin that contains the selected value. The final loss varies by at most $0.015$ nats across the wide interval $\eta \in [0.10, 0.99]$, and degrades smoothly at the limit ($\eta \approx 0$) while the pure-delta-rule limit ($\eta \to 1$) showed the best final validation loss by a small amount. \Cref{tab:eta-sweep} shows that $\eta = 0.99$ yields the lowest final loss in this $5{,}440$-step ablation, ahead of the $0.30$ and $0.50$ grid points that bracket the selected value by $0.014$ and $0.015$ nats and ahead of the next best point by $0.009$ nats. Both margins are no larger than the $0.0150$-nat seed standard deviation we measure on the 67M loss gap in \Cref{tab:scale}, and this ablation is single-seed, so we do not interpret the $\eta = 0.99$ advantage as significant. This insensitivity is the practical correlate of the Tikhonov interpolation of \Cref{thm:fixedpoint}, where a wide range of $\eta$ produces buffers that interpolate between the EMA and Wiener targets without sharply changing the optimization behavior.

The ablation is run at the configuration of \Cref{tab:hps} with everything except $\eta$ held fixed, so the sweep isolates $\eta$ and any offset common to the row cancels in the spread. Its grid brackets the selected $\eta = 0.37$ between the $0.30$ and $0.50$ points, both of which sit in the flat basin. This is what makes the recipe of Appendix~\ref{app:hp-sweep} safe to state as the interval $[0.3, 0.5]$ rather than as a point.

\section{Broader Impacts}
\label{app:broader-impacts}

\DeltaMomentum is a general-purpose optimization method, not a deployed system.
Its potential positive impact is improved training efficiency, which may reduce the
compute cost of neural-network training and make experimentation more accessible.
Its potential negative impact is indirect, since more efficient optimization could also lower
the cost of training models for harmful or dual-use applications. The work does not
release a high-risk model or dataset.

\end{document}